\documentclass[10pt,journal]{IEEEtran}
\usepackage{eurosym}
\usepackage{graphicx}
\usepackage{amsmath}
\usepackage{amssymb}
\usepackage{subfigure}
\usepackage{amsfonts}
\usepackage{algorithm}
\usepackage{algpseudocode}
\usepackage{epstopdf}
\usepackage{balance}
\usepackage{placeins}
\usepackage{booktabs}

\newtheorem{theorem}{Theorem}
\newtheorem{lemma}{Lemma}
\input{tcilatex}
\graphicspath{{figures/}}

\begin{document}

\title{Deep Convolutional Tables:\\
Deep Learning without Convolutions}
\author{Shay Dekel, Bar Ilan University Ramat-Gan, Israel \\
Yosi Keller, Bar Ilan University Ramat-Gan, Israel\\
Aharon Bar-Hillel Ben-Gurion University, Negev, Israel \thanks{%
All authors contributed equally.} }
\maketitle

\begin{abstract}
We propose a novel formulation of deep networks that do not use dot-product
neurons and rely on a hierarchy of voting tables instead, denoted as
Convolutional Tables (CT), to enable accelerated CPU-based inference.
Convolutional layers are the most time-consuming bottleneck in contemporary
deep learning techniques, severely limiting their use in Internet of Things
and CPU-based devices. The proposed CT performs a fern operation at each
image location: it encodes the location environment into a binary index and
uses the index to retrieve the desired local output from a table. The
results of multiple tables are combined to derive the final output. The
computational complexity of a CT transformation is independent of the patch
(filter) size and grows gracefully with the number of channels,
outperforming comparable convolutional layers. It is shown to have a better
capacity:compute ratio than dot-product neurons, and that deep CT networks
exhibit a universal approximation property similar to neural networks. As
the transformation involves computing discrete indices, we derive a soft
relaxation and gradient-based approach for training the CT hierarchy. Deep
CT networks have been experimentally shown to have accuracy comparable to
that of CNNs of similar architectures. In the low compute regime, they
enable an error:speed trade-off superior to alternative efficient CNN
architectures.
\end{abstract}

\markboth{Journal of \LaTeX\ Class Files,~Vol.~14, No.~8, August~2015}{Shell
\MakeLowercase{\textit{et al.}}: Bare Demo of IEEEtran.cls for Computer
Society Journals}%
\IEEEtitleabstractindextext{

}

\IEEEdisplaynontitleabstractindextext
\IEEEpeerreviewmaketitle

\section{Introduction}

\label{sec:introduction}

Deep learning techniques in general, and convolutional neural networks
(CNNs) in particular, have become the core computational approach in
computer vision. Currently, `deep learning techniques' and `deep neural
networks' are synonyms, despite relating to different concepts. By
definition~\cite{Deng2014}, deep learning is the \textquotedblleft use of
multiple layers to progressively extract higher-level features from raw
input,\textquotedblright\ i.e., learning a useful feature hierarchy.
Artificial neural networks, on the other hand, are graphs of
dot-products/convolutions computing elements. In this work, we argue that it
is possible to unravel this dualism, and that this is a worthwhile endeavor.

It is the need for accelerated and more efficient CPU-based inference that
has motivated us to depart from the dot-product/convolution-based paradigm.
CNNs consist of millions of `neurons': computing elements applying dot
products to their inputs. The resulting computational load is significant
and is usually handled by GPUs and other specialized hardware. Due to these
computational demands, deep learning applications are often restricted to
the server-side, where powerful dedicated hardware carries the computational
load. Alternatively, applications running on standard CPUs or low-computing
platforms such as cellphones are limited to small, low-accuracy networks, or
avoid using networks altogether. Such applications are natural user
interfaces (NUIs), face recognition, mobile robotics, and Internet of Things
(IoT) devices, to name a few. Low-compute platforms are not equipped with a
GPU and are expected to handle the inference of simple deep learning tasks.
Such devices are unable to train CNNs or run ImageNet-scale CNNs. Other
devices of interest, such as most laptops and desktop computers, are
unequipped with GPUs but have gigabytes of memory.

In this work, we propose a novel deep learning paradigm based on a hierarchy
of deep voting tables. This allows the embedding of input signals such as
images, similarly to backbone CNNs, which is shown to be preferable in
low-compute and GPU-less domains. Our approach uses CPUs that are capable of
efficiently accessing large memory domains (random access). For instance,
for a CPU with ideally unbounded memory access, the fastest classifier is a
single huge table: a series of simple binary queries is applied to the
input, an index is built from the resulting bits, and an answer is retrieved
from the table entry. Such an approach is impractical due to the huge size
of the resulting table. A practical alternative is to replace the single
table with an ensemble of smaller hierarchical tables. Rather than rely on
dot products, we propose a two-step, computationally efficient alternative.
First, a set of $K$ simple binary queries is applied to the input patch,
thus encoding it with a binary codeword of $K$-bit. Second, the computed
codeword is used as an index into a table to retrieve the corresponding
output representation. The proposed transformation is called a Convolutional
Table (CT). A CT layer consists of multiple CTs whose outputs are summed.
Similarly to convolutional layers, CT layers are stacked to form a CT
network and derive a hierarchy of features. We detail this construction in
Section~\ref{sec:Model}.

A fern is a tree in which all split criteria at a fixed tree depth are
identical, such that the leaf identity is determined by using fixed $K$
split criteria termed 'bit-functions'. The CT operation is the application
of a single fern to a single input location, and the operation of a CT layer
(sum of CTs), corresponds to applying a fern ensemble. In Section~\ref%
{sec:complexity}, we compare the computational complexity of the CNN and CT
operations. By design, the complexity of a single CT layer is independent of
the size of the patch used (`filter size' in a convolutional layer), and its
dependence on quadratic depth terms is with significantly lower constants.
This allows us to derive conditions for which a CT network can be
considerably faster than a corresponding CNN.

In Section~\ref{sec:Theory}, we derive rigorous results regarding the
capacity of a single fern and the expressiveness of a two-layer CT network.
Specifically, the VC dimension of a single fern with $K$ bits is shown to be
$\Theta (2^{K})$. This implies that a fern has an advantage over a dot
product operation: it has a significantly higher
capacity-to-computing-effort ratio. Our results show that a two-layered CT
network has the same universal approximation capabilities as a two-layered
neural network, and that this also holds for ferns with a single
bit-function per fern.

A core challenge of the proposed deep-table approach is the need for a
proper training process. The challenge arises because inference requires
computing discrete indices that are unnameable for gradient-based learning.
Thus, we introduce in Section~\ref{sec:Optimization} a `soft formulation of the
calculation of the $K$ bit codeword', as a Cartesian product of the $K$ soft bit
functions. Following \cite{Krupka2014,Shotton2011}, bit-functions are
computed by comparing only two pixels in the region of interest. However, in
their soft version, the pixel indices parameters are non-integer, and their
gradient-based optimization relies on the horizontal and vertical spatial
gradients of the input maps. During training, soft indices are used for
table voting using a weighted combination of possible votes. We use an
annealing mechanism and an incremental learning schedule to gradually turn
trainable but inefficient soft voting into efficient, hard single-word
voting at the inference time.

We developed a CPU-based implementation of the proposed scheme and applied
it to multiple standard data-sets, consisting of up to 250K images. The
results, presented in Section~\ref{sec:Experiments}, show that CT-networks
achieve an accuracy comparable to CNNs with similar structural parameters
while outperforming the corresponding CNNs trained using weights and
activations binarization. More importantly, we consider the speed:error and
speed:error:memory trade-offs, where speed is estimated via operation count.
CT networks have been shown to provide better speed:error trade-off than
efficient CNN architectures such as MobileNet~\cite{sandler2018mobilenetv2}
and ShuffleNet~\cite{ma2018shufflenet} in the tested domain. The improved
speed:error trade-off requires a higher memory consumption, as CT networks
essentially trade speed for memory. Overall, CT networks can provide $%
3.1-12X $ acceleration over CNNs with memory expansions of $2.3X-10.6X$,
which are applicable for laptops and IoT applications.

In summary, our contribution in this paper is two-fold: First, we present an
alternative to convolution-based networks and show for the first time that
useful deep learning (i.e., learning of a feature hierarchy) can be achieved
by other means. \ Second, from an applicable viewpoint, we suggest a
framework enabling accelerated CPU inference for low-compute domains, with
moderate costs of accuracy degradation and memory consumption.

\section{Related Work}

\label{sec:Related work}

Computationally efficient CNN schemes were extensively studied using a
plethora of approaches, such as low-rank or tensor decomposition~\cite%
{Jaderberg2014}, weight quantization~\cite%
{AlHami2018TowardsAS,Krishnamoorthi2018QuantizingDC}, conditional
computation ~\cite{Davis2013LowRankAF,Bengio2015ConditionalCI,Waissman2018},
using FFT for efficient convolution computations~\cite{Mathieu2014}, or
pruning of networks weights and filters~\cite%
{NIPS1989_250,han2015deep_compression,HashedNets,7472157}. In the following, we focus on the methods most related to our work.

Binarization schemes aim to optimize CNNs by binarizing network weights \cite%
{Martinez2020Training,Bi_Real_Net,BinaryConnect} and internal network
activations \cite{XNOR-Net,LinZP17}. Recent approaches apply
weights-activations-based binarization \cite{XNOR_plus}, to reduce memory
footprint and computational complexity by utilizing only binary operations.
Methods which binarize the activations are related to the proposed scheme,
as they encode each local environment into a binary vector. However, their
encoding is dense, and the output is computed using a (binary) dot product.
The proposed scheme also utilizes binary activations, but differs
significantly, as convolution/dot product operations are not applied. In
contrast, we apply binary comparisons between random samples in the
activation maps.

Knowledge distillation approaches~\cite{Hinton2015,Yim_2017_CVPR} use a
large \textquotedblleft Teacher\textquotedblright\ CNN \ to train a smaller
computationally efficient \textquotedblleft Student\textquotedblright\ CNN.
We show in our experiments that this technique can be combined with our
approach to improve the accuracy. Regarding network design, several
architectures were suggested for low-compute platforms~\cite%
{sandler2018mobilenetv2,ma2018shufflenet,wu2018shift, ephrath2019leanresnet}%
, utilizing group convolutions~\cite{sandler2018mobilenetv2}, efficient
depth-wise convolutions~\cite{sandler2018mobilenetv2,ma2018shufflenet}, map
shifts~\cite{wu2018shift}, dubbed Bi-Real net \cite{Bi_Real_Net} and sparse
convolutions~\cite{ephrath2019leanresnet}. Differentiable soft quantization
was proposed by Gong et al. \cite{DSQ_method} to quantize the weights of a
CNN to a given number of bits. A paper by Huang et al. \cite%
{EfficientQuantization} proposed a method for efficient quantization using a
multilevel binarization technique and a linear or logarithmic
quantizer to simultaneously binarize weights and quantize
activations to a low bit width. The authors of the paper then compare the
trade-offs of this method to low-compute versions of MobileNetV2 \cite%
{sandler2018mobilenetv2}, ShuffleNet \cite{ma2018shufflenet} and Bi-Real Net
\cite{Bi_Real_Net}. The use of trees or ferns applied using convolutions for
fast predictors was thoroughly studied in computer vision \cite{Shotton2011,
Krupka2014,Ren2015,Krupka2017}. A convolutional random forest enabled
the estimation of human pose in real time in the Kinect console~\cite{Shotton2011}. A
flat CT ensemble classifier for hand poses, that was proposed in~\cite%
{Krupka2014}, enabled inference in less than $1$millisecond on a CPU. It was
extended \cite{Krupka2017} to a full-hand pose estimation system. However,
in all these works, flat predictors were used, while we extend the notion to
a deep network with gradient-based training. Other works merged tree
ensembles with MLPs or CNNs for classification~\cite%
{Kontschieder2015,Ioannou2016}, or semantic segmentation~\cite{Bulo2014}.
Specifically, `conditional networks' were presented in~\cite{Ioannou2016} as
CNNs with a tree structure, enabling conditional computation to improve the
speed-accuracy ratio. Unlike our approach, where ferns replace the
convolutions, these networks use standard convolutions and layers, and the
tree/forest is a high-level routing mechanism. Yang et al. proposed a
probability distribution approach for quantization \cite{SLB} by treating
the discrete weights in a quantized neural network as searchable variables
and using a differential method to search for them accurately. They
represent each weight as a probability distribution over the set of discrete values. While this method was able to produce quantized neural networks with
higher performance than other state-of-the-art quantized\textbf{\ }methods
in image classification, it is still reliant on standard convolution layers
with dot products, which could be a bottleneck in low-power systems. In the
experimental results section, we compare our method with that of the authors and report superior results in terms of accuracy. An
effective channel pruning method was suggested by Xiu Su et al \cite%
{SearchingBilaterally}, using a bilaterally coupled network to determine the
optimal width of a neural network, resulting in a more compact network. A
bilaterally coupled network is a network that has been trained to
approximate the optimal width of another network, to achieve a more compact
network while maintaining or improving its performance. A new approach for
dynamically removing redundant filters from a neural network was proposed by
Yehui Tang et al. ~\cite{ManifoldRegularized}. The channels are pruned by
embedding the manifold information of all instances into the space of pruned
networks. The performance of deep convolutional neural networks is improves
by removing redundant filters while preserving the important features of the
input data. The authors show that their method is an alternative to traditional
channel pruning methods, which may not be able to achieve optimal results in
terms of both accuracy and efficiency. A scheme related to our approach was
proposed by Zhou et al.~\cite{Zhou2017}, wherein the authors suggest a
cascade of forests applied using convolutions. Each forest is trained to solve
the classification task, and the class scores of a lower layer forest are
used as features for the next cascade level. While the classifiers'
structure is deep, they are not trained end-to-end, the trees are trained
independently using conventional random forest optimization. Hence, the
classifier uses thousands of trees and is only able to reach the accuracy of
a two-layer CNN.

\begin{figure}[t]
\includegraphics[width=0.95
\linewidth]{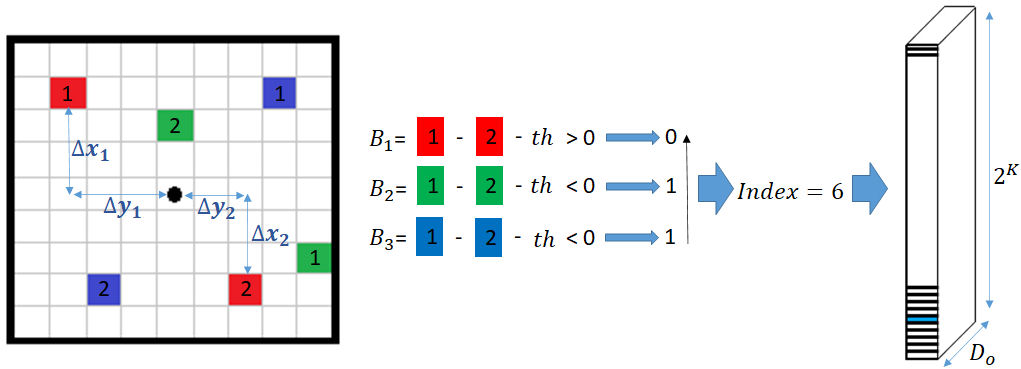}
\caption{A Convolutional Table (CT) with a word calculator of three bits
operating on a 2D input tensor. Each bit-function is given by two pixels
whose difference is compared to a threshold (colored in red, green, and
blue). The word computed by the three bits, '110' = 6, is used to access the
$6^{th}$ entree in the voting table, thus computing an output vector $\in
\mathbb{R}
^{D_{o}}$.}
\label{fig:bit_fun}
\end{figure}

\section{Convolutional Table Transform}

\label{sec:Model}

A convolutional table is a transformation accepting a representation tensor $%
T^{in}\in \mathbb{R}^{H_{i}\times W_{i}\times D_{i}}$ and outputting a
representation tensor $T^{out}\in \mathbb{R}^{H_{o}\times W_{o}\times D_{o}}$%
, where $H_{i},H_{i}$ and $W_{i},W_{0}$ are the spatial dimensions, while $%
D_{i}$, $D_{o}$ are the representation's dimensionality for the input and
output tensors, respectively. Formally, it is the tuple $(B,W)$, where $B $
is a word calculator, and $W$ is a voting table. A word calculator $B$ is a
feature extractor applied to a patch, returning a $K$-bit index, such that $%
B:\mathbb{R}^{l\times l\times D_{i}}\rightarrow \{0,1\}^{K}$, where $l$ is
the patch size. The voting table $W\in M^{_{2^{K}\times D_{o}}}$ contains
all possible $2^{K}$ outputs as rows, each being a vector $\in
\mathbb{R}
^{D_{o}}$.

The index is computed by sequentially applying $K$ bit-functions $%
\{B^{k}\}_{k=1}^{K}$, each computing a single bit of the index. Let $%
P=(p_{x},p_{y})\in \mathbb{R}^{H_{i}\times W_{i}}$ be a location in $T^{in}$%
. We denote by $B(P;\Theta )=(B^{1}(P;\Theta ^{1}),...,B^{K}(P,\Theta ^{K}))$
the index computed for the patch centered at location $P$. Here $\Theta
=(\Theta ^{1},...,\Theta ^{K})$ are the word calculator parameters, with $%
\Theta ^{k}$ explicitly characterized below. The functions $B^{k}(P;\Theta
^{k})\in \{0,1\}$ (dependence on $T^{in}$ is omitted for notation brevity)
are computed by thresholding at $0$ the simple smooth functions $%
b^{k}(P;\Theta ^{k})$ of a few input pixels. (i.e., $B^{k}(P;\Theta
)=Q\left( b^{k}(P;\Theta ^{k})\right) $ with $Q(\cdot )$ being the Heaviside
function). Bit functions $b^{k}$ are the thresholded differences between
two locations:
\begin{eqnarray}
b^{k}\left( P;\Theta _{k}\right)  &=&T^{in}\left( p_{x}+\Delta
x_{1}^{k},p_{y}+\Delta y_{1}^{k},c^{k}\right)   \label{eq:BitFunctions} \\
&&-T^{in}\left( p_{x}+\Delta x_{2}^{k},p_{y}+\Delta y_{2}^{k},c^{k}\right)
-th^{k},  \notag
\end{eqnarray}%
where the learned parameters $\Theta ^{k}$\textbf{, }define the compared
locations $(\Delta x_{1}^{k},\Delta y_{1}^{k},c^{k})$\textbf{\ }and $(\Delta
x_{2}^{k},\Delta y_{2}^{k},c^{k})$, and the margin threshold $th^{k}$. $th^{k}$ is the learnable scalar margin threshold for comparing
two locations within the patch corresponding to the $k$'th bit. The codeword
computed by the word calculator is used as an index to the voting table to
retrieve $W(B(P;\Theta ),:)\in
\mathbb{R}
^{D_{o}}$ for location $P$. Figure~\ref{fig:bit_fun} depicts the
convolutional table operation at a single location. A CT (using padding=
`same' and stride=$1$) applies the operation defined above at each position $%
P\in \mathbb{R}^{H_{i}\times W_{i}}$ of the input tensor to obtain the
corresponding output tensor column. Same as in convolutional layers, padding
can be applied to the input for computations at border pixels, and strides
may be used to reduce the spatial dimensions of the output tensor. A CT
layer includes $M$ convolutional tables $\{(B_{m},W_{m})\}_{m=1}^{M}$ whose
outputs are summed:
\begin{eqnarray}
T^{out}[P, &:&]=\sum_{m=1}^{M}W_{m}[B_{m}(T_{N(P)}^{in}),~:~],
\label{eq:Tout} \\
B_{m} &=&\{B^{k}\}_{k=1}^{K}  \notag
\end{eqnarray}

The operation of a CT layer is summarized in Algorithm \ref{alg:Deep_CT}.
Essentially, it applies a fern ensemble at each location, summing the output
vectors from the ensemble's ferns. A CT network stacks multiple CT layers in
a cascade or graph structure, similar to stacking convolutional layers in
standard CNNs.

\begin{algorithm}[t]
\caption{Convolutional Table Layer}
\begin{algorithmic}[1]
\State \textbf{INPUT:} \par {A tensor $T^{in} \in \mathbb{R}^{ H_i\times W_i \times D_i}$,\par transformation parameters $\{(\Theta^{m},W^{m})\}_{m=1}^{M}$}
\State  \textbf{OUTPUT:} \par{ An Output tensor $T^{out} \in \mathbb{R}^{ H_o\times W_o \times D_o}$}
\State Initialization: \par $T^{out}= 0$, \par $\{(\Theta^{m}, W^{m})\}_{m=1}^{M} \sim \mathcal{N}(0,\sigma^{2}=1)$
\For{$ P  \in \{1,..,H_i\}\times\{1,..,W_i\}$ } \Comment{ All pixel locations}
\For{$m=1,..,M$}    \Comment{ All CT tables}
\State{Compute $B_m(P;\Theta^{k,m} ) \in {\{0,1\}}^K$ } \Comment{Eq.~\ref{eq:BitFunctions}}
\State{$T^{out}[P,:]=T^{out}[P,:]+W^{m}[B_m(P;\Theta^{k,m} ),:]$} \Comment{Eq.~\ref{eq:Tout}}
\EndFor
\EndFor
\State{ Return $T^{out}$ }
\end{algorithmic}
\label{alg:Deep_CT}
\end{algorithm}

\section{Computational complexity}

\label{sec:complexity}

Let the input and output tensors be $T^{in}\in \mathbb{R}^{H_{i}\times
W_{i}\times D_{i}}$ and $T^{out}\in \mathbb{R}^{H_{o}\times W_{o}\times
D_{o}}$, respectively, where the dimensions are the same as in Section \ref%
{sec:Model}, and let $l\times l$ be the filter's support. For a convolution
layer, the number of operations per location is $C_{CNN}=l^{2}D_{i}D_{o}$,
since it computes $D_{o}$ inner products at the cost of $l^{2}D_{i}$ each.
For a CT layer with $M$ CTs and $K$ bits in each CT, the cost of bits
computation is $MKC_{b}$, with $C_{b}$ being the cost of computing a single
bit-function. The voting cost is $MD_{o}$, since we add the vectors $\in
\mathbb{R}^{MD_{o}}$ to get the result. Thus, the complexity of a CT layer
is $C_{CT\;}=M(C_{b}\cdot K+D_{o})$.

It follows that the computational complexity of a CT is independent of the
filter size $l$. Hence, it allows agglomerating evidence from large spatial
regions with no additional computational cost. To compare CNN and CT
complexities, assume that both use the same representation dimension $%
D_{i}=D_{o}=D$. The relation between the total number of bit-functions $MK$
and $D$ depends on whether each bit-function uses a separate dimension, or
the dimensions are reused by multiple bit-functions. For a general analysis,
assume each dimension is reused by $R$ bit-functions, such that $D=MK/R$ (in
our experiments $R\in \lbrack 1,3]$). Thus, the ratio of complexities
between CNNs and CT layers with the same representation dimension $D$ is
\begin{equation}
\frac{C_{CNN}}{C_{CT\;}}=\frac{l^{2}D^{2}}{M(C_{b}\cdot K+D)}=\frac{%
l^{2}D^{2}}{C_{b}DR+\frac{D^{2}R}{K}}\approx \frac{Kl^{2}}{R},
\label{equ:complexity2}
\end{equation}%
\textbf{\ }where the last approximation is true for large $D$, and the
quadratic terms in $D$ are dominant. The computation cost of computing a
single bit function, $C_{b}$, is defined as the number of operations
required to perform a single comparison operation between two locations
within the patch size. This computation is described in equation \ref%
{eq:BitFunctions} and involves four additions of the offsets of $\Delta _{x}$%
, $\Delta {y}$, and a subtraction of the threshold, $th^{k}$. This results
in a total of five operations. Additionally, load operations are required
for the two relevant memory locations and their two base addresses, bringing
the total number of operations to nine. Although the most stringent count does
not exceed 10 operations, for simplicity, we set $C{b}$ to 10. The ratio of
complexities between CNNs and CT layers in Eq. \ref{equ:complexity2}
suggests there is the potential for acceleration by more than an order of
magnitude for reasonable choices of parameter values. For example, using $l=3
$, $K=8$, and $R=2$ results in a $36$ $\times$ acceleration. Although the
result of this analysis is promising, several lower-level considerations
are essential to achieve actual acceleration. First, operation counts
translate into actual improved speed only if the operations can be
efficiently parallelized and vectorized. In particular, vectorization
requires contiguous memory access in most computations. The CT
transformation adheres to these constraints. Contrary to trees, bit
computations in ferns can be vectorized over adjacent input locations as all
locations use the same bit-functions, and the memory access is contiguous.
The voting operation can be vectorized over the output dimensions. This
implementation was already built and has been shown to be highly efficient for flat
CT classifiers \cite{Krupka2017}. Although a CT transformation uses a
`random memory access' operation when the table is accessed, this random
access is rare: there is a single random access operation included in the $%
(C_{b}\cdot K+D_{o})$ operations required for a single fern computation.
Another important issue is the need to keep the model's total storage size
within reasonable bounds to enable its memory to be efficiently accessed in
a typical $L_{2}$ cache. For example, a model consisting of $50$ layers,
wherein $M=10$ ferns, $K=6,$ and $D=60$, have a total size of $1.92MB$ when
the parameters are stored in $8$-bit precision.

\section{CT capacity and expressiveness}

\label{sec:Theory}

In this section, we present two analytic insights that indicate the feasibility
of deep CT networks and their potential advantages. In Section~\ref%
{subsec:Capacity} the capacity of a fern is shown to be $\Omega (2^{K})$
with $O(K)$ computational effort, providing a much better capacity:compute
ratio than a linear convolution. In Section~\ref{subsec:Universality} a
network of two fern layers is shown to be a universal approximation,
similarly to known results on neural networks. In the analysis, we consider
a simplified non-convolutional setting, with an input vector $X\in \lbrack
0,1]^{d}$, and a non-convolutional fern ensemble classifier. Furthermore,
simple bit-functions of the form $B^{k}(X)=Q\left(
(-1)^{s^{k}}(X[I^{k}]-t^{k})\right) $ are considered, where $s^{k}\in
\{0,1\},~I^{k}\in {1,..,d}~,t^{k}\in \lbrack 0,1]$. These simple
bit-functions just compare a single entry to a threshold.

\subsection{Capacity:Compute ratio}

\label{subsec:Capacity}

The following lemma characterizes the capacity of a single fern in
VC-dimension terms.

\begin{lemma}
\label{lemmaVC} Define the hypothesis family of binary classifiers based on
a single $K-bit$ fern $H=\{f(X)=sign(W(B(X))):B={\{(s^{k},I^{k},t^{k})\}}%
_{k=1}^{K},W\in R^{2^{K}}\}$. For $K\geq 4$ and $log_{2}(d)<K\leq d$ its
VC-dimension is
\begin{equation}
2^{K}\leq VC-dim(H)\leq K2^{K}
\end{equation}
\end{lemma}

\begin{proof}
Since we only consider the sign of the chosen table entry $W(B(X))$, we may
equivalently consider the family of classifiers $f(X)=W(B(X))$ for $W\in {%
\{-1,1\}}^{2^{K}}$.

\textbf{Lower bound:} Consider a sample containing points on the $d$%
-dimensional cube%
\begin{equation}
S=\{P_{c}=(-1^{c_{1}},...-1^{c_{K}}):~c=(c_{1},..,c_{K})\in \{0,1\}^{K}~\}.
\end{equation}%
It follows that this sample can be shattered by the binary fern family. We
may choose the bit functions $b^{k}(X)=Q(X[k])$. Assume an arbitrary label
assignment $l$, i.e. for each example $P_{c}$ we have $l(P_{c})=y_{c}$ with $%
y_{c}$ arbitrarily chosen in $\{-1,1\}$. By definition, for each $c$ and $k$%
, we have $b^{k}(P_{c})=\{-1\}^{c_{k}}=P_{c}[k]$. Since any two points $%
P_{c_{1}},P_{c_{2}}$ differ by at least a single dimension, they will have at least
different bits in a single bit function. Hence, the $2^{K}$ points
are mapped into the $2^{K}$ different cells of $W$. By choosing $W[c]=y_{c}$
we may get the desired labeling $l$. Therefore, $S$ is shattered, showing $%
VCdim(H)\geq 2^{K}$.

\textbf{Upper Bound:} Assume $K\geq 4$ and $log_{2}(d)<K\leq d$. We compute
an upper bound on the number of possible labels (label vectors) enabled by
the binary fern family on a sample of size $n$, and show it is smaller than $%
2^{n}$ for $n=K2^{K}$.

For a $K$-bit fern, there are $d^{K}$ ways to choose the input dimension of
the $K$ bit-functions (one dimension per function). For each bit-function,
once the input dimension $I$ is chosen, we may re-order the examples
according to their value in dimension $I$, i.e., $X_{1}[I]\leq
X_{2}[I]...\leq X_{n}[I]$. Non-trivial partitions of the sample, i.e.
partitions into two nonempty sets, are introduced by choosing the threshold $%
t$ in $(X_{1}[I],X_{n}[I])$, and there are at most $n-1$ different options.
Including the trivial option of partitioning $S$ into $S$ and $\Phi $, there
are $n$ options.

Consider the number of partitions of the examples into $2^{K}$ cells, which
are made possible using $K$ bit-functions. Two different partitions must
differ w.r.t. the induced partition by a single bit-function choice, at
least. Hence, their number is bounded by the number of ways to choose the
bit-function partitions, which is $d^{K}n^{K}$ at most, according to the
above considerations. For each such partition into $2^{K}$ cells, $2^{K}$
different classifiers can be defined by choosing the table entries $%
\{W[c]\}_{c=0}^{2^{K}-1}$. Hence, the total number of possible different
fern classifiers for $n$ points is bounded by $(dn)^{K}2^{2^{K}}$. When this
bound is lower than $2^{n}$, the sample cannot be shattered. Thus, we solve the problem
\begin{equation}
(dn)^{K}2^{2^{K}}<2^{n}
\end{equation}%
and
\begin{equation}
2^{K}+Klog_{2}d+Klog_{2}n<n.  \label{Condition1}
\end{equation}%
By choosing $n=K2^{K}$ it follows that Eq.~\ref{Condition1} is satisfied
since
\begin{multline}
2^{K}+Klog_{2}d+Klog_{2}K+K^{2} \\
<2^{K}+3K^{2}\leq 2^{K}+(K-1)2^{K}\leq K2^{K}
\end{multline}%
For the second inequality we used $log_{2}d<K$ (assumed), and for the third
we used the inequality $3K^{2}\leq (K-1)2^{K}$ which holds for $K\geq 4$.
\end{proof}

Results similar to the Lemma~\ref{lemmaVC} were derived for trees~\cite%
{Mansour1997,Yildiz2015}, and the lemma implies that the capacity of ferns
is not significantly lower than that of trees. To understand its implications,
compare such a single-fern classifier to a classifier based on a single
dot-product neuron, i.e., a linear classifier. A binary dot-product neuron
with $d=K$ input performs $O(K)$ operations to obtain a VC dimension of $O(K)$%
. The fern also performs $O(K)$ computations, but the response is chosen
from a table of $2^{K}$ leaves, and the VC dimension is $\Omega (2^{K})$.
The higher (capacity)\textbf{:}(computational-complexity) ratio of ferns
indicates that they can compute significantly more complex functions than
linear neurons using the same computational budget.

\subsection{CT networks are universal approximators}

\label{subsec:Universality}

The following theorem states the universal approximation capabilities of a $%
2 $-layer CT network. It resembles the known result for two-stage
neural networks~\cite{Hassoun1995}.

\begin{theorem}
Any continuous function $f:[0,1]^{d}\rightarrow \lbrack 0,1]$ can be
approximated in $L_{\infty }$ using a two-layer fern network with $K=1$ in
all ferns.
\end{theorem}

\begin{proof}
The core of the proof is to show that by using $2K$ 1-bit ferns at the first
layer, a second layer fern can create a function with an arbitrary value
inside a hyper-rectangle of choice, and zero outside the rectangle. Since
sums of such `step functions' are dense in $C[0,1]$, so are two-layer fern
networks.

Let $R=[a_{1},b_{1}]\times \lbrack a_{2},b_{2}]...\times \lbrack
a_{d},b_{d}] $ be a hyper rectangle in $R^{d}$ and $v\in R$ a scalar value.
A rectangle function $s(x;v,R):[0,1]^{d}\rightarrow R$ is defined as $v\cdot
1_{x\in R}$, i.e. a function whose value is $f(x)=v$ for $x\in R$ and $0$
otherwise. Define the family of step functions
\begin{equation}
G=\{g:[0,1]^{d}\rightarrow \lbrack
0,1]:g(x)=\sum_{p=1}^{P}s(x;v_{p},R_{p})\}.
\end{equation}%
It is known that $G$ is dense in $(C[0,1]^{d})$ (from the Stone-Weierstrass
theorem, see e.g.~\cite{eidelman2004functional} ). We will show that the
family of $2$-layer CT networks includes this set.

Let $s(v,R)$ be an arbitrary rectangle function. For each dimension $%
i=1,..,d $, define the following two bit-functions: $L_{i}(x)=Q(x-a_{k})$
and $R_{i}(x)=Q(-(x-b_{k}))$. Denote these bit-functions by $B_{(R)}^{j}$
for $j=1,..,2d$. By construction, ${\{B_{R}^{j}\}}_{j=1}^{2d}$ characterize
the rectangle%
\begin{equation}
x\in R\iff \forall j=1,..,2d~~~B_{R}^{j}(x)=1.
\end{equation}%
Equivalently, we have $x\in R$ iff $\sum_{j=1}^{2d}B_{R}^{j}(x)>2d-\frac{1}{2%
}$. Given a function $g(x)=\sum_{p=1}^{P}s(x;v_{p},R_{p})$ to implement, we
define at the first layer $P$ sets of ferns. Fern set $p$ contains $2d$
ferns, denoted by $F_{p}^{j}$, with a single bit-function each, and the
bit-function of $F_{p}^{j}$ is $B_{R_{p}}^{j}$. The output dimension of the
first layer is $P$. The weight table $W_{p}^{j}\in M_{2\times P}$ of $%
F_{p}^{j}$ is defined by
\begin{equation}
W_{p}^{j}[i,l]=%
\begin{cases}
1 & i=1,l=p \\
0 & otherwise%
\end{cases}%
.
\end{equation}%
The output vector of a layer is the sum of all ferns $Y=\sum_{p^{\prime
}=1}^{P}\sum_{j=1}^{2d}W_{p^{\prime }}^{j}[B_{R_{p^{\prime }}}^{j}(x),:].$
By construction we have
\begin{eqnarray}
Y[p] &=&\sum_{p^{\prime }=1}^{P}\sum_{j=1}^{2d}W_{p^{\prime
}}^{j}[B_{R_{p}^{\prime }}^{j}(x),p] \\
&=&\sum_{j=1}^{2d}W_{p}^{j}[B_{R_{p}}^{j}(x),p]=%
\sum_{j=1}^{2d}B_{R_{p}}^{j}(x).  \notag
\end{eqnarray}%
The second equality holds since $W_{p^{\prime }}^{j}[i,p]$ only differs from
zero for $p=p^{\prime }$. The last equality is valid since $W_{p}^{j}[i,p]=i$
for $i=0,1$. We hence have $Y[p]=2d>2d-\frac{1}{2}$ for $x\in R_{p}$, that
has lower values anywhere else.

The second layer of ferns contains $P$ ferns, each with a single bit and a
single output dimension. For fern $p$, we define the bit-function $%
B_{p}^{(2)}(Y)=Q(Y[p]-(2d-\frac{1}{2}))$ which fires only for $x\in R_{p}$
and the table $W_{p}^{(2)}=[0,v(p)]$. The output unit $U$ computes
\begin{eqnarray}
U &=&\sum_{p=1}^{P}W_{p}^{(2)}[B_{p}^{(2)}(x)] \\
&=&\sum_{p=1}^{P}v(p)\cdot B_{p}^{(2)}(x)=\sum_{p=1}^{P}v(p)\cdot 1_{x\in
R_{p}}.  \notag
\end{eqnarray}%
Therefore, $U$ implements the function $g(x)$, which completes the proof.
\end{proof}

\section{Training with soft Convolutional Tables}

\label{sec:Optimization}

Following Eq.~\ref{eq:BitFunctions}, $B^{k}\left( P;\Theta _{k}\right)
=Q(b^{k}(P;\Theta _{k}))$ does not enable gradient-based optimization, as
the derivative of the Heaviside function $Q(\cdot )$ is zero almost
anywhere. We suggest a soft version of the CT in Section~\ref{subsec:forward}
that enables gradient-based optimization, and discuss its gradient
in~Section \ref{subsec:backward}.

\subsection{A soft CT version}

\label{subsec:forward} To enable gradient-based learning, we suggest replacing the Heaviside function during optimization with a linear sigmoid $%
q(x;t)$ for $t>0$:
\begin{equation}
q(x;t)=min(max(~(t+x)/2t~,0),1).  \label{eq:softIndicator}
\end{equation}

$q(x;t)$ is identical to the Heaviside function for $x\gg 0$ and is a linear
function in the vicinity of $x=0$. It has a nonzero gradient in $[-t,t]$,
where $t$ is a hyperparameter controlling its smoothness. For low $t$
values, we have $q(x;t)\underset{t\rightarrow 0}{\rightarrow }Q(x)$.
Moreover, we have $q(x;t)+q(-x;t)=1$. Thus, it can be interpreted as a
pseudo-probability, with $q(x)$ estimating the probability of a bit being $1$
and $q(-x)$ the probability of being 0.

Following Section~\ref{sec:Model}, a word calculator $B$ maps a patch $P$ to
a single index or equivalently to a one-hot vector (row) in $R^{2^{K}}$.
Denote the word calculator in the latter view (mapping into $R^{2^{K}})$
by $\vec{B}(P;\Theta )$. Given this notation, the CT output is given by the
product $\vec{B}\cdot W$. Hence, extending the word calculator to a soft
function in $R^{2^{K}}$, with dimension $b\in \{0,..,2^{K}-1\}$ measuring
the activity of the word $b$, provides a natural differentiable extension of
the CT formulation. For an index $b,$ define the sign of its $k^{th}$ bit by
$s(b,k)\triangleq (-1)^{(1+u(b,k))}$, where $u(b,k)$ denotes the $k^{th}$
bit of the index $b$ in its standard binary expansion, such that%
\begin{equation}
s(b,k)=\left\{
\begin{tabular}{rl}
$1$ & bit $k=1$ \\
$-1$ & bit $k=0$%
\end{tabular}%
\right. .
\end{equation}%
The activity level of the word $b$ in a soft word calculator is defined by%
\begin{equation}
\vec{B}^{s}(P;\Theta )[b]=\prod_{k=1}^{K}~q(~s(b,k)\cdot b_{k}(P;\Theta
_{k});t~).  \label{eq:SoftFrensBitsSign}
\end{equation}

Intuitively, the activity level of each possible codeword index $b$ is a
product of the activity levels of the individual bit-functions in the
directions implied by the codeword bits. Marginalization implies that $%
\sum_{b=0}^{2^{K}-1}\vec{B}^{s}(P)[b]=1$, and hence the soft-word calculator
defines a probability distribution over the possible $K$-bit words. Soft
CT is defined as natural extension $\vec{B}^{s}W$. Note that as $%
t\rightarrow 0$, the sigmoid $q(x;t)$ becomes sharp, $\vec{B}^{s}$ becomes a
one-hot vector and the soft CT becomes a standard CT as defined in Section~%
\ref{sec:Model}, without any inner products.

Soft CT can be considered as a consecutive application of two layers: a
soft indexing layer $\vec{B}^{s}$ followed by a plain linear layer (though
sparse if most of the entries in $\vec{B}^{s}$ are zero). Since CT is
applied to all spatial locations on the activation map, the linear layer $W$
corresponds to a $1\times 1$ convolution. We implemented a sparse $1\times 1$
convolutional layer operating on a sparse tensor $\mathbf{x\in }R^{H\times
W\times 2^{K}}$ and outputting a dense tensor $\mathbf{y\in }$ $R^{H\times
W\times D_{o}}$.

\subsection{Gradient computation and training}

\label{subsec:backward}

Since applying the soft CT in a single position is $W\cdot B^{s}(P;\Theta ),$
with $W$ being a linear layer, it suffices to consider the gradient of $\vec{%
B}^{s}(P;\Theta )$. Denote the neighborhood $l\times l$ of the location $P$ in $%
T^{in}$ by $T_{N(P)}^{in}$. Taking into account a single output variable $b$ at a
time, the gradient with respect to the input patch $T_{N(P)}^{in}$ is given
by
\begin{equation}
\frac{\partial \vec{B}^{s}(P)[b]}{\partial T_{N(P)}^{in}}=\sum_{k}\frac{%
\partial \vec{B}^{s}(P)[b]}{\partial b_{k}(P;\Theta _{k})}\cdot \frac{%
\partial b_{k}(P;\Theta _{k})}{\partial T_{N(P)}^{in}},
\end{equation}%
with a similar expression for the gradient with respect to $\Theta $. For a
fixed index $k$, the derivative $\frac{\partial \vec{B}^{s}(P)[b]}{\partial
b_{k}(P;\Theta _{k})}$ is given by
\begin{multline}
\frac{\partial \vec{B}^{s}(P)[b]}{\partial b_{k}(P;\Theta _{k})}= \\
\prod_{j\neq k}~q(s(b,j)b_{j}(P;\Theta _{j});t)\frac{\partial
q(s(b,j)b_{j}(P;\Theta _{j});t)}{\partial b_{k}(P;\Theta _{k})}= \\
\frac{\vec{B}^{s}(P)[b]\cdot t^{-1}\cdot 1_{|b_{k}|<t}\cdot s(b,k)}{%
2q(~s(b,k)\cdot b_{k}(P;\Theta _{k}))}.~~~
\end{multline}%
The derivative is non-zero when the word $b$ is active (i.e. $\vec{B}%
^{s}(P)[b]>0$), and the bit-function $b_{k}(P)$ is in the dynamic range $%
-t<b_{k}(P)<t$.

To further compute the derivatives of $b_{k}(P)$ with respect to the input
and parameters ($\frac{\partial b_{k}(P;\Theta _{k})}{dT_{N(P)}^{in}}$ and $%
\frac{\partial b_{k}(P;\Theta _{k})}{d\Theta }$, respectively), we first
note that the forward computation of $b_{k}(P)$ estimates the input tensor
at fractional spatial coordinates. This is due to the fact that the offset parameters $%
\Delta x_{1},\Delta y_{1},\Delta x_{2},\Delta y_{2}$ are trained with
gradient descent, which requires them to be continuous. We use bilinear
interpolation to compute image values at fractional coordinates in the
forward inference. Hence, such bilinear interpolation (of spatial gradient
maps) is required for gradient computation. For instance, $\frac{\partial
b_{k}(P)}{\partial \Delta x_{1}^{k}}$ is given by%
\begin{multline}
\frac{\partial b_{k}(P)}{\partial \Delta x_{1}^{k}}=\frac{\partial
T^{in}\left( p_{x}+\Delta x_{1}^{k},p_{y}+\Delta y_{1}^{k},c\right) }{%
\partial \Delta x_{1}^{k}}  \label{eq:bitFuncGrad} \\
=\frac{\partial T^{in}(:,:,c)}{\partial x}[(p_{x}+\Delta
x_{1}^{k},p_{y}+\Delta y_{1}^{k})],
\end{multline}%
where $\frac{\partial T^{in}(:,:,c)}{\partial x}$ is the partial
x-derivative of the image channel $T^{in}(:,:,c)$, that is estimated
numerically and sampled at $\left( x,y\right) =(p_{x}+\Delta
x_{1}^{k},p_{y}+\Delta y_{1}^{k})$ to compute Eq.~\ref{eq:bitFuncGrad}. The
derivatives with respect to other parameters are computed in a similar way. Note
that the channel index parameters $c^{k}$ are not learnt - these are fixed
during network construction for each bit-function. As for $\frac{\partial
b_{k}(P;\Theta _{k})}{dT_{N(P)}^{in}}$, by considering Eq.~\ref%
{eq:BitFunctions}, it follows that $b_{k}(P;\Theta _{k})$ only relates to
two fractional pixel locations: $\left( p_{x}+\Delta x_{1}^{k},p_{y}+\Delta
y_{1}^{k}\right) $ and $\left( p_{x}+\Delta x_{2}^{k},p_{y}+\Delta
y_{2}^{k}\right) $. This implies that a derivative with respect to eight
pixels is required, as each value of a fractional coordinates pixel is
bilinearly interpolated using its four neighboring pixels in integer
coordinates.

The sparsity of the soft word calculator (both forward and backward) is
governed by the parameter $t$ of the sigmoid in Eq.~\ref{eq:softIndicator},
which acts as a threshold. When $t$ is large, most of the bit-functions are
ambiguous and soft i.e., not strictly $0$ or $1$, and the output will be
dense. We can control the output's sparsity level by adjusting $t$. In
particular, as $t\rightarrow 0,$ the bit-functions become hard, and $%
B^{s}(P) $ converges toward a hard fern with a single active output word.
While a dense flow of information and gradients is required during training,
fast inference in test time requires a sparse output. Therefore, we introduce an
annealing scheduling scheme, such that $t$ is initiated as $t\gg 0$, set to
allow a fraction $f$ of the values of the bit functions to be in the `soft zone' $%
[-t,t]$. The value of $t$ is then gradually lowered to achieve a sharp and
sparse classifier towards the end of the training phase.

\section{Experimental Results}

\label{sec:Experiments}

We experimentally examined and verified the accuracy and computational
complexity of the proposed CT scheme, applying it to multiple small and
medium image datasets, consisting of up to 250K images. Such datasets are
applicable to IoT and non-GPU low compute devices. We explored the
sensitivity of CT networks to their main hyperparameters in Section~\ref%
{sec:ablation_study}. To this end, we studied the number of ferns $M$ and the
number of bit functions per fern $K$, and the benefits of distillation as a
CT training technique. We then compared our CT-based architectures to
similar CNN classes. Such a comparison is nontrivial due to the inherent
architectural differences between the network classes. Thus, in Section~\ref%
{subsec:SimilarRepDim} we compared Deep CT to CNNs with similar depth
and width, that is, the number of layers and maps per layer. The comparison focuses
on efficient methods suggested for network quantization and binarization, since
the CT framework uses similar notions. In Section \ref{subsec:TradeOffs} the
CT is compared to CNN formulations with comparable computational complexity,
i.e., similar number of MAC (Multiply-ACcumulate) operations. We compared
CT with CNN architectures specifically designed for inference efficiency,
such as MobileNet~\cite{sandler2018mobilenetv2} and ShuffleNet~\cite%
{ma2018shufflenet}. The comparison was applied by studying the speed:error
and speed:error:memory trade-offs of the different frameworks. Finally, to
exemplify the applicability of Deep CT to low-power Internet-of-Things (IoT)
applications and larger datasets (250K images), it was applied to IoT-based
face recognition in~Section \ref{sec:WebCasia}.
\begin{figure*}[tbh]
\centering%
\begin{tabular}{ccc}
\subfigure[]{\includegraphics[height=0.23%
\linewidth]{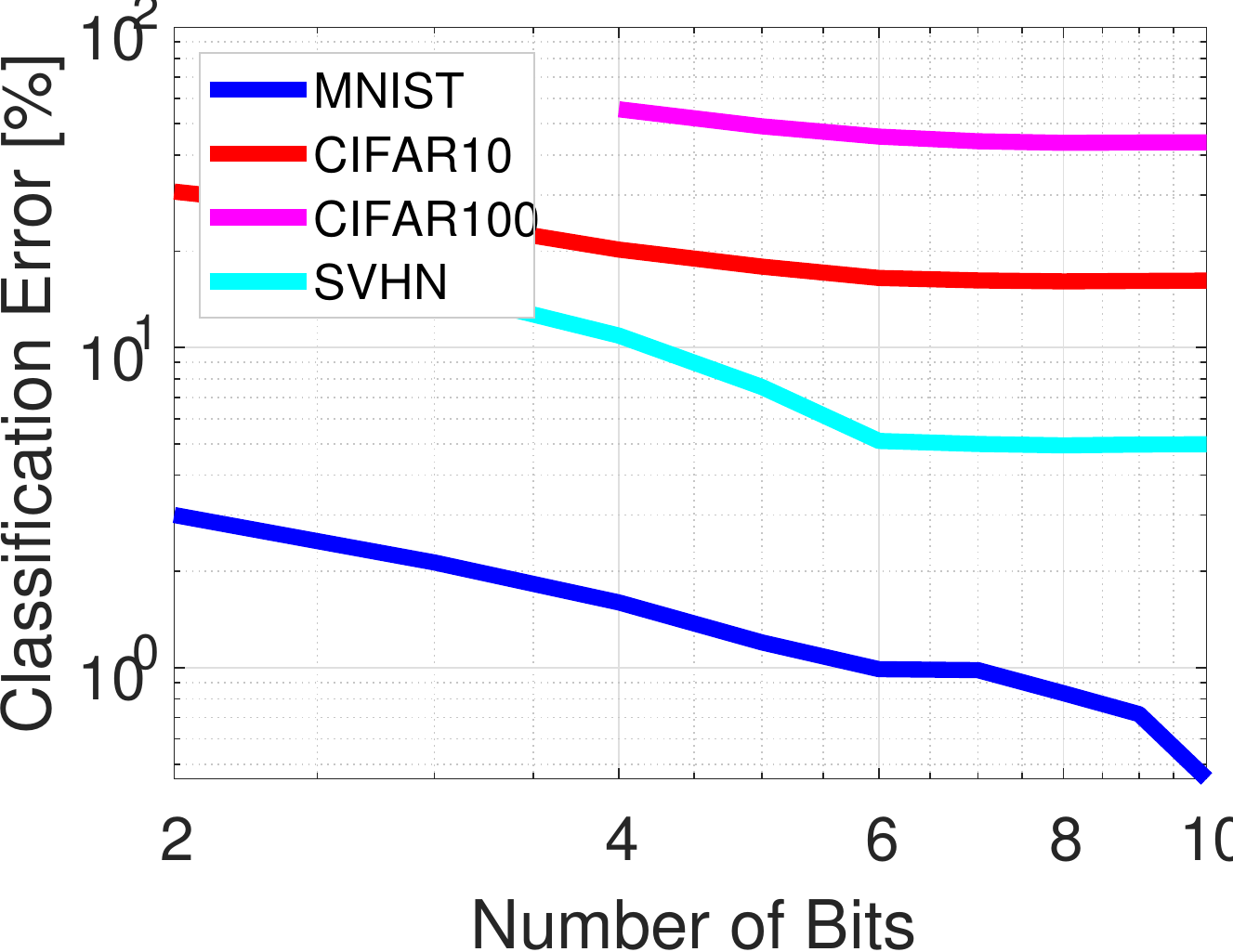}} & \subfigure[]{%
\includegraphics[height=0.23\linewidth]{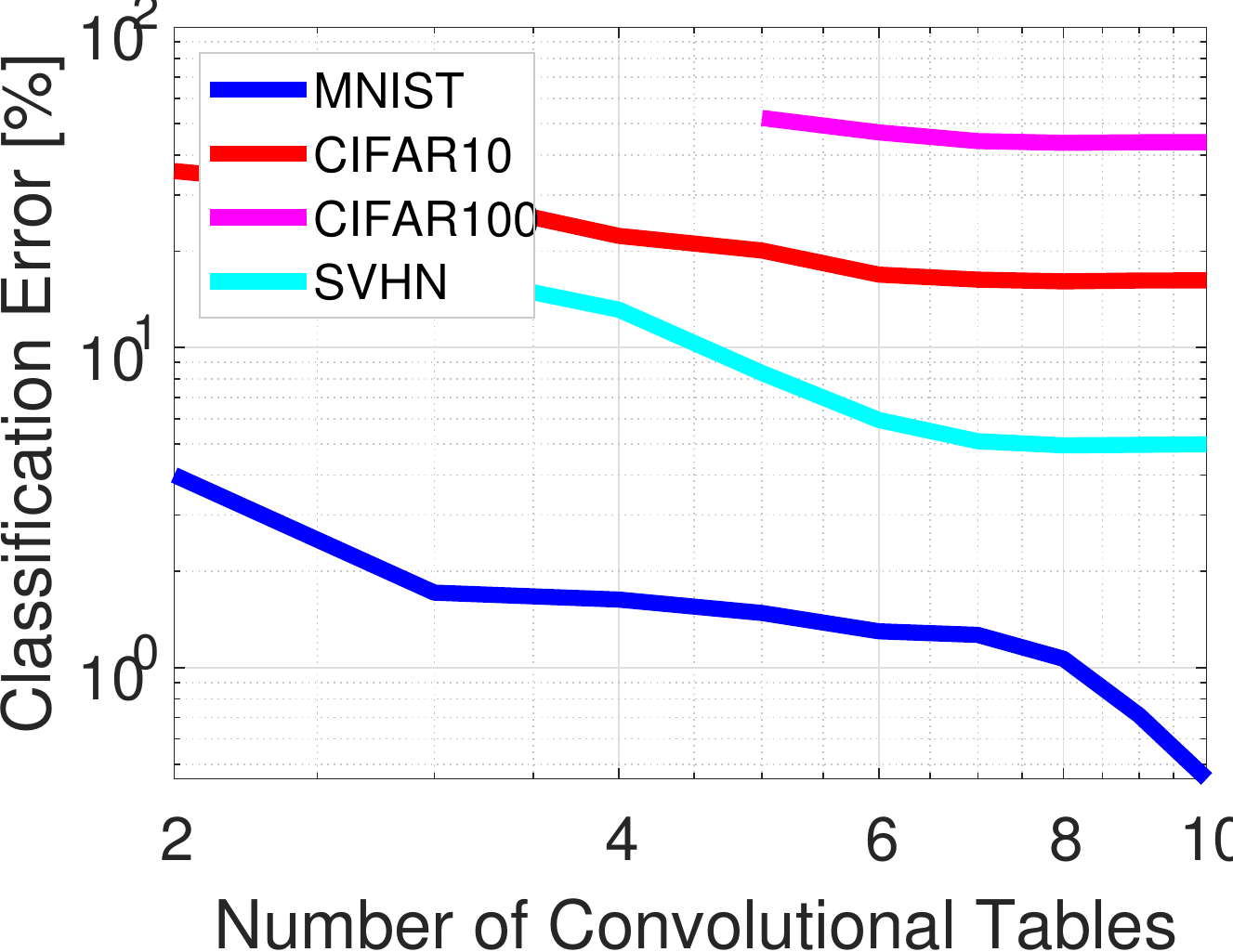}} &
\subfigure[]{\includegraphics[height=0.23
\linewidth]{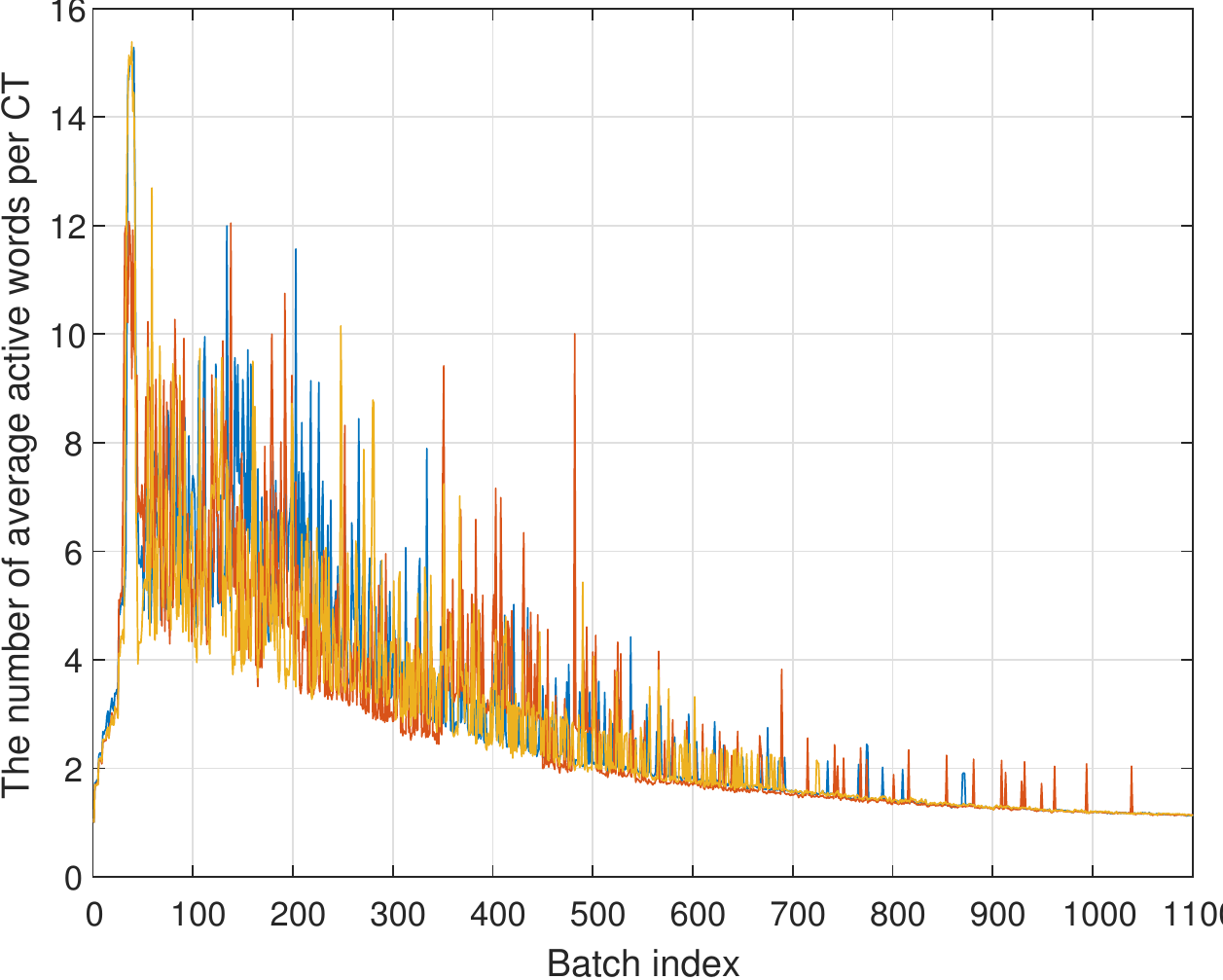}}%
\end{tabular}%
\caption{Hyperparameters study on MNIST, CIFAR10, CIFAR100, and SVHN
datasets. \textbf{(a)} Deep CT network error Vs. the number of bits $K$ in a
log-log plot. In all layers, the number of convolutional tables $M$ was
fixed to $10$. \textbf{(b)} Deep CT network error Vs. the number of
convolutional tables $M$ in a log-log plot, for a fixed number of bits $K=10$%
. \textbf{(c)} The average number of active words per CT was evaluated using
the CIFAR-10 dataset and plotted against the training batch index. Each
convolution table is represented by a different color: yellow, blue, and
red. As expected, the number of active words decreases during the training
process due to the annealing mechanism and eventually converges to a single
active word.}
\label{fig:DeepCtAblation}
\end{figure*}

The proposed novel CT layer does not use standard CNN components such as
activations, convolutions, or FCs. Thus, the Deep CT word calculator layer
and the sparse voting table layer were implemented from scratch, in
unoptimized C, compiled as MEX files, and applied within a Matlab
implementation. Our experiments were carried out on the same computer using
CPU-only implementations for all schemes.

Similar to other inherently novel layers, such as Spatial Transformer
Networks, it is challenging to implement a converging end-to-end training
process. To facilitate the convergence of the proposed scheme, a three-phase
training scheme was applied. First, the lower half of the CNN (layers
\#1-\#6 in Table \ref{table:webcasia_arch}) was trained, by adding the
Average Pooling (AP), Softmax (SM) and Cross Entropy (CE) layers on top of
layer \#6. Then, the full architecture was trained with layers \#1 through
\#6 frozen. Finally, following Hinton et al. \cite{Hinton2015}, we unfroze all
layers and refined the entire network using a distillation-based approach.
The CT was trained to optimize a convex combination of CE loss and
Kullback-Leibler divergence with respect to the probabilities induced by a
teacher CNN. The temperature parameter was initially set at $t=4$ and
gradually annealed to $t=1$ toward the end of the training. All Deep CT
networks were trained from scratch using random Gaussian initialization of
the parameters.

\subsection{Hyperparameters sensitivity}

\label{sec:ablation_study}

The expressive power of the Deep CT transformation is related to the number
of CTs per layer $M$, and the number of bit-functions $K$ used in each CT.
We tested the sensitivity of the Deep CT network to these parameters and the
suggested annealing schedule. These experiments applied CT networks with
four layers to the SVHN, CIFAR-10, and CIFAR-100 datasets, and two layers to
the MNIST dataset. The Networks are similar to those shown in the appendix,
in Tables~\ref{table:NetworkArch_2} and \ref{table:NetworkArch_4}. We used a
fixed patch size $l=5$ and input and output dimensions of $D_{i}=D_{o}=32$,
respectively. It is best to apply CT networks with a patch size larger than
the kernel size of CNN-based networks, since CT patches are sparsely
represented. Comparatively, CNN complexity is quadratic with respect to
patch size, while CT networks' computational complexity is independent of
patch size.

The classification errors with respect to the bit-functions $K$ (for a fixed
$M$), and the number of CTs $M$ (for a fixed $K$) are shown in Figs.~\ref%
{fig:DeepCtAblation}a and ~\ref{fig:DeepCtAblation}b, respectively. For most
datasets, there is a range in which the error exponentially decreases as the
number of network parameters increases, while the accuracy improvement
saturates for $M>8$ and $K>8$ for the deeper networks (datasets other than
MNIST).

Figure \ref{fig:DeepCtAblation}c shows the average number of active words in
a convolutional table as a function of the training time for CIFAR-10. The
fraction $f$ of ambiguous bits (bits whose value is not set to either $0$ or
$1$) was started at $0.2$ and was gradually reduced exponentially, while the
sigmoid threshold $t$ was adjusted periodically according to the required $f$%
. In our experiments, the sigmoid threshold was set to $t=0.5$. The number
of active words per CT is thus gradually reduced and converged to a single
active word in the final classifier. The CNN with four layers obtained a
test error of 12.55\%, compared to 12.49\% with the same architecture
without annealing. Therefore, annealing-based training led to a sparser
classifier without significantly degrading the classification accuracy.

To establish the best hyperparameters, additional experiments were carried
out for multiple M and K combinations around $M=8,K=8$. The experiments were
conducted using the CIFAR-10 dataset, and the CT networks having four and
six layers, detailed in Tables~\ref{table:NetworkArch_4} and \ref%
{table:NetworkArch_6}, respectively. In addition, we examined training with
and without distillation from a teacher. The $M=8,K=8$ configuration was
found to be the most accurate. Distillation was found to be beneficial, providing
accuracy improvements of $1\%-2\%$.
\begin{table}[tbh]
\centering%
\begin{tabular}{ccccc}
\hline
$\mathbf{K}$ & \multicolumn{1}{|c}{$\mathbf{M}$} & \multicolumn{1}{|c}{%
\textbf{Layers\#}} & \multicolumn{1}{|c}{\textbf{Distil}} &
\multicolumn{1}{|c}{\textbf{Error [\%]}} \\ \hline\hline
9 & 8 & 6 & No & 11.31 \\
8 & 9 & 6 & No & 12.31 \\
\textbf{8} & \textbf{8} & 6 & No & \textbf{11.05} \\
7 & 8 & 6 & No & 13.93 \\
8 & 7 & 6 & No & 13.34 \\
\textbf{8} & \textbf{8} & 6 & Yes & \textbf{10.25} \\ \hline
9 & 8 & 4 & No & 12.88 \\
8 & 9 & 4 & No & 13.05 \\
\textbf{8} & \textbf{8} & 4 & No & \textbf{12.49} \\
7 & 8 & 4 & No & 13.25 \\
8 & 7 & 4 & No & 12.90 \\
\textbf{8} & \textbf{8} & 4 & Yes & \textbf{10.91} \\ \hline
\end{tabular}%
\caption{Joint hyperparameters tuning for the four and six layers CT
networks described in Tables~\protect\ref{table:NetworkArch_4} and \protect
\ref{table:NetworkArch_6}. $M$ and $K$ denote the Convolution tables' size
and the number of bits, respectively. }
\label{fig:ablation_study_table}
\end{table}
\begin{table*}[tbh]
\centering%
\begin{tabular}{cl|cc|ccc}
\hline
\multicolumn{2}{c|}{\textbf{Deep CT}} & \textbf{Parameters } &
\multicolumn{1}{|c|}{\textbf{Output}} & \textbf{Convolution} &
\multicolumn{1}{|c}{\textbf{Parameters}} & \multicolumn{1}{|c}{\textbf{Output%
}} \\
\multicolumn{2}{c|}{\textbf{Layer}} & $l,K,M,D_{o}$ & \multicolumn{1}{|c|}{}
& \textbf{Layers} & \multicolumn{1}{|c}{} & \multicolumn{1}{|c}{} \\
\hline\hline
1 & CT Layer 1 & $7,8,8,32$ & $26\times 26\times 32$ & \multicolumn{1}{|l}{
Conv,ReLU} & $l=3,D_{o}=32$ & $30\times 30\times 32$ \\
2 & Average Pooling & $l=3$ & $24\times 24\times 32$ & \multicolumn{1}{|l}{
Max Pooling} & $l=3,Stride=2$ & $14\times 14\times 32$ \\
3 & CT Layer 2 & $5,8,8,32$ & $20\times 20\times 32$ & \multicolumn{1}{|l}{
Conv,ReLU} & $l=3,D_{o}=32$ & $12\times 12\times 32$ \\
4 & Average Pooling & $l=3$ & $18\times 18\times 32$ & \multicolumn{1}{|l}{
Max Pooling} & $l=2$ & $11\times 11\times 32$ \\
5 & CT Layer 3 & $5,8,8,64$ & $14\times 14\times 64$ & \multicolumn{1}{|l}{
Conv,ReLU} & $l=3,D_{o}=64$ & $9\times 9\times 64$ \\
6 & Average Pooling & $l=3$ & $12\times 12\times 64$ & \multicolumn{1}{|l}{
Max Pooling} & $l=2$ & $8\times 8\times 64$ \\
7 & CT Layer 4 & $5,8,8,128$ & $8\times 8\times 128$ & \multicolumn{1}{|l}{
Conv,ReLU} & $l=3,D_{o}=128$ & $6\times 6\times 128$ \\
8 & Average Pooling & $l=2$ & $7\times 7\times 128$ & \multicolumn{1}{|l}{
Max Pooling} & $l=2$ & $5\times 5\times 128$ \\
7 & CT Layer 5 & $3,8,8,64$ & $5\times 5\times 64$ & \multicolumn{1}{|l}{
Conv,ReLU} & $l=3,D_{o}=64$ & $4\times 4\times 64$ \\
8 & Average Pooling & $l=2$ & $4\times 4\times 64$ & \multicolumn{1}{|l}{Max
Pooling} & $l=2$ & $3\times 3\times 64$ \\
9 & CT Layer 6 & $3,8,8,10$ & $2\times 2\times 10$ & \multicolumn{1}{|l}{
Conv,ReLU} & $l=3,D_{o}=10$ & $1\times 1\times 10$ \\
10 & Average Pooling & $l=2$ & $1\times 1\times 10$ & \multicolumn{1}{|l}{FC}
& $D_{o}=10$ & $1\times 1\times 10$ \\
11 & SoftMax &  &  & \multicolumn{1}{|l}{SoftMax} &  & $1$ \\ \hline
\end{tabular}%
\caption{The six-layer Deep CT and CNN networks applied to the CIFAR10,
CIFAR100, and SVHN datasets. The CNN binarization schemes were applied to
the CNN architecture, where $l$ and $K$ denote the patch size and the number
of bits, $M$ is the number of convolutional tables and $D_{o}$ is the number
of output maps.}
\label{table:NetworkArch_6}
\end{table*}
\begin{table*}[tbh]
\centering%
\begin{tabular}{lcccccccccc|cc}
\hline
\textbf{Dataset} & \multicolumn{1}{|c}{\textbf{DeepCT}} &
\multicolumn{1}{|c}{\textbf{BC}} & \multicolumn{1}{|c}{\textbf{XN}} &
\multicolumn{1}{|c}{\textbf{XN++}} & \multicolumn{1}{|c}{\textbf{TABCNN}} &
\multicolumn{1}{|c}{\textbf{DSQ}} & \multicolumn{1}{|c}{\textbf{DSQ}} &
\multicolumn{1}{|c}{\textbf{DSQ}} & \multicolumn{1}{|c|}{\textbf{Bi-Real}} &
\multicolumn{1}{|c|}{\textbf{SLB}} & \textbf{DeepCT } & \multicolumn{1}{|c}{%
\textbf{CNN}} \\
& \multicolumn{1}{|c}{} & \multicolumn{1}{|c}{\textbf{\cite{BinaryConnect}}}
& \multicolumn{1}{|c}{\textbf{\cite{XNOR-Net}}} & \multicolumn{1}{|c}{%
\textbf{\cite{XNOR_plus}}} & \multicolumn{1}{|c}{\textbf{\cite{LinZP17}}} &
\multicolumn{1}{|c}{\textbf{\cite{DSQ_method} 2b}} & \multicolumn{1}{|c}{%
\textbf{\cite{DSQ_method} 4b}} & \multicolumn{1}{|c}{\textbf{\cite%
{DSQ_method} 6b}} & \multicolumn{1}{|c|}{\textbf{\cite{Bi_Real_Net}}} &
\textbf{\cite{SLB} } & \textbf{+ distil.} & \multicolumn{1}{|c}{} \\
\hline\hline
\multicolumn{10}{c}{\textbf{Two layers [\%]}} &  &  &  \\ \hline\hline
\textbf{MNIST} & \textbf{0.35} & 0.61 & 0.51 & 0.39 & 0.49 & 0.41 & 0.38 &
0.36 & 0.36 & 0.37 & 0.39 & 0.334 \\ \hline
\multicolumn{10}{c}{\textbf{Four layers [\%]}} &  &  &  \\ \hline\hline
\textbf{CIFAR10} & \textbf{12.49} & 14.47 & 14.77 & 12.97 & 13.97 & 12.87 &
12.71 & 12.68 & 12.72 & 12.66 & 10.91 & 10.86 \\
\textbf{CIFAR100} & \textbf{39.17} & 50.27 & 47.32 & 41.03 & 45.32 & 40.14 &
40.05 & 39.96 & 39.97 & 39.55 & 37.12 & 36.81 \\
\textbf{SVHN} & \textbf{4.75} & 6.21 & 5.88 & 4.92 & 5.11 & 4.85 & 4.79 &
4.77 & 4.79 & 4.80 & 4.46 & 4.27 \\ \hline
\multicolumn{10}{c}{\textbf{Six layers [\%]}} &  &  &  \\ \hline\hline
\textbf{CIFAR10} & \textbf{11.05} & 12.56 & 12.91 & 11.88 & 12.13 & 11.35 &
11.27 & 11.14 & 11.19 & 11.16 & 10.25 & 9.88 \\
\textbf{CIFAR100} & \textbf{37.95} & 42.43 & 40.14 & 39 & 39.37 & 38.31 &
38.25 & 38.12 & 38.15 & 38.05 & 36.05 & 35.58 \\
\textbf{SVHN} & \textbf{4.12} & 4.95 & 4.54 & 4.25 & 4.28 & 4.21 & 4.18 &
4.15 & 4.17 & 4.25 & 3.96 & 3.77 \\ \hline
\end{tabular}%
\caption{Comparison of the classification error percentage. We compare the
results for CNNs consisting of 2 layers (MNist) or 4 and 6 layers (other
datasets). All quantized/binarized CNNs were derived by applying the
corresponding quantization/binarization schemes to the baseline CNN having
the same number of layers (Tables~\protect\ref{table:NetworkArch_2}, \protect
\ref{table:NetworkArch_4} and \protect\ref{table:NetworkArch_6}). The most
accurate results are marked in \textbf{bold}.}
\label{table:DCT2-results}
\end{table*}

\subsection{Accuracy comparison of low-complexity architectures}

\label{subsec:SimilarRepDim}

Although deep CT and CNN networks use very different computational
elements, both use a series of intermediate representations to perform
classification tasks. Hence, we consider CT and CNN backbone networks
similar if they have the same number of intermediate representations (depth)
and the same number of maps in each representation (width). In this section,
we report experiments comparing CNN and CT networks that have identical
configurations in these respects. Although the intermediate spatial tensors
were chosen to be similar, we allowed for different top inference layers. Those
were chosen to optimize the classification accuracy for each network class.
Thus, the CNN-based networks used max-pooling and an FC layer, while the CT
CNN only used an average pooling layer.

We compared the classification accuracy of CT networks to standard CNNs and
contemporary low-complexity CNN architectures based on CNN binarization
schemes. The methods we compared against include the BinaryConnect (BC) ~%
\cite{BinaryConnect} \footnote{%
https://github.com/MatthieuCourbariaux/BinaryConnect}, XNOR-Net (XN) \cite%
{XNOR-Net} \footnote{%
https://github.com/rarilurelo/XNOR-Net}, XNOR-Net++ (XN++)\cite{XNOR_plus},
Bi-Real Net (Bi-Real) \cite{Bi_Real_Net}, SLB \cite{SLB} \footnote{%
https://github.com/zhaohui-yang/Binary-Neural-Networks/tree/main/SLB} and
TABCNN\cite{LinZP17}\footnote{%
https://github.com/layog/Accurate-Binary-Convolution-Network}, using their
publicly available implementations. These methods, \cite{XNOR-Net} and \cite%
{LinZP17} in particular, also encode the input activity using a set of
binary variables. In contrast to our CT table-based approach, they use dense
encodings and apply a binary dot-product to their encoded input. Finally, we
applied the state-of-the-art Differentiable Soft Quantization (DSQ) scheme
\cite{DSQ_method}\footnote{%
https://github.com/ricky40403/DSQ} to the conventional CNN to quantize it to
2, 4, and 6 bits.

All networks were applied to the MNIST \cite{lecun2010mnist}, CIFAR-10 \cite%
{CIFAR10}, CIFAR-100 \cite{CIFAR100}, and SVHN \cite{SVHN} datasets. For the
MNIST dataset, we applied CT networks and CNNs consisting of two layers, as
this relatively small dataset did not require additional depth. For the
other datasets, networks consisting of four and six layers were applied.

The architectures of the six-layer networks are detailed in Table~\ref%
{table:NetworkArch_6}, and the architectures of the two and four layers
used are detailed in the appendix in Tables~\ref{table:NetworkArch_2} and %
\ref{table:NetworkArch_4}, respectively. The classification accuracy results
are reported in Table~\ref{table:DCT2-results}. On the left-hand side of the
table, all methods are compared using the cross-entropy loss. The proposed
CT networks outperformed the binarization and discretization schemes in all
cases and were outperformed by conventional CNNs. In the right-hand
column of the table, we show the results of a CT network using a
distillation loss and a VGG16~\cite{Zhang2015} teacher network. These CT
network improve the accuracy significantly further, and the margin of CNNs
over them is also significantly smaller. Several key findings can be drawn
from these results: First, the good accuracy of CT networks indicates that
they can be trained using standard SGD optimization, such as CNNs, without
significant overfitting. Second, the similar precision of the CT and CNN networks
with a similar intermediate representation structure indicates that the accuracy
is more dependent on the representation hierarchy structure than the
particular computing elements used.
\begin{figure*}[tbh]
\centering%
\begin{tabular}{cc}
\subfigure[]{\includegraphics[height=0.30%
\linewidth]{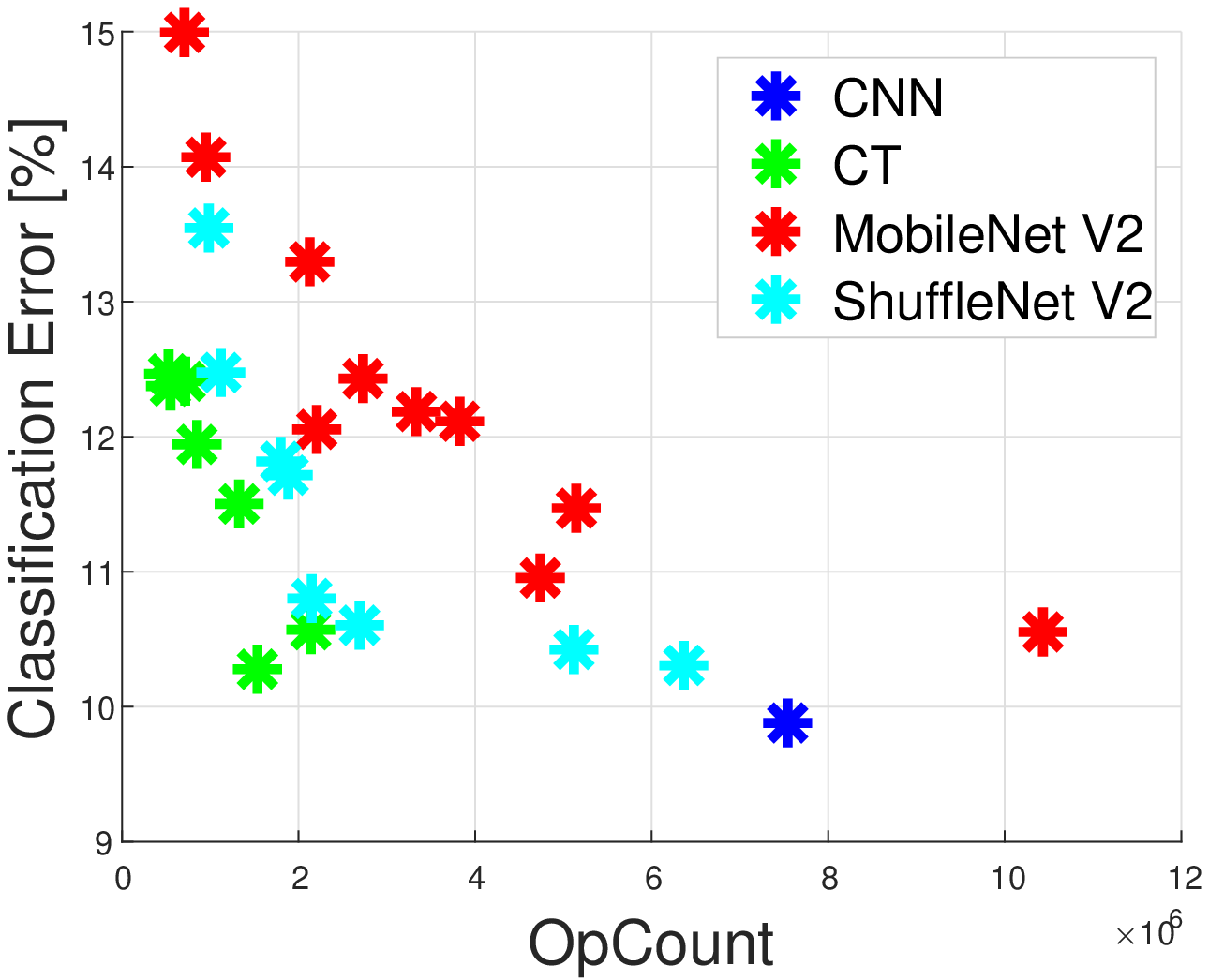}} & \subfigure[]{%
\includegraphics[height=0.30%
\linewidth]{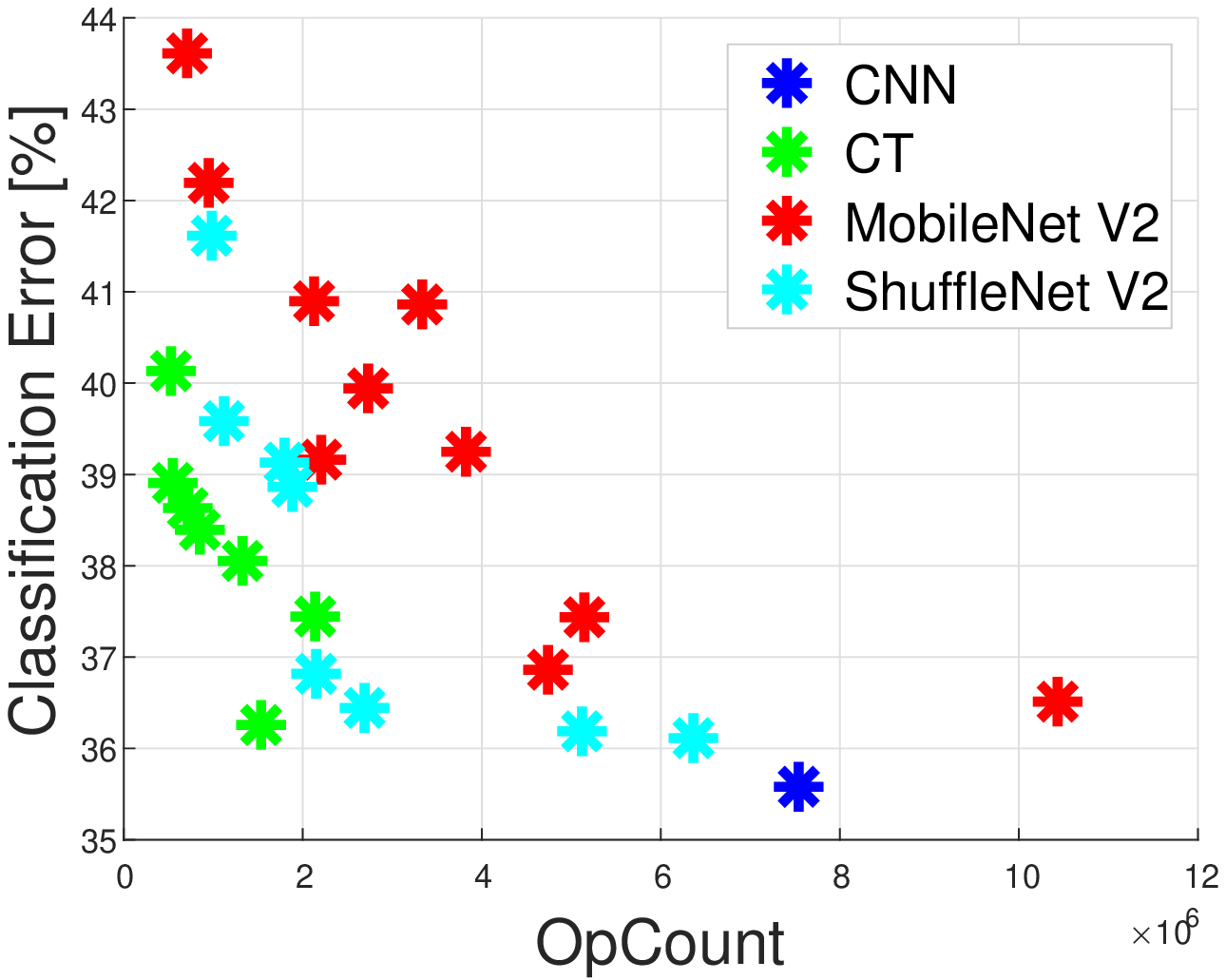}} \\
\subfigure[]{\includegraphics[height=0.30%
\linewidth]{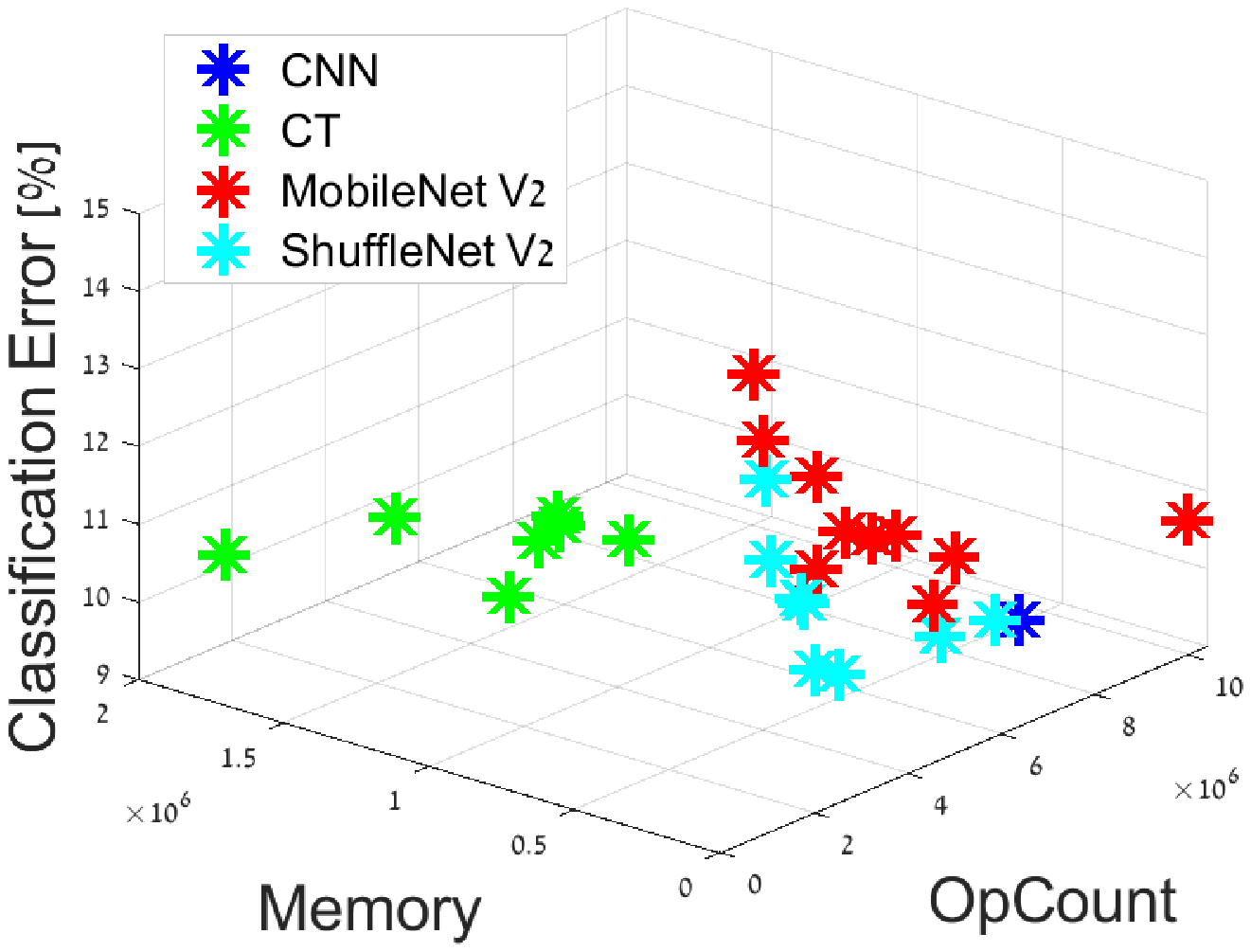}} & %
\subfigure[]{\includegraphics[height=0.30%
\linewidth]{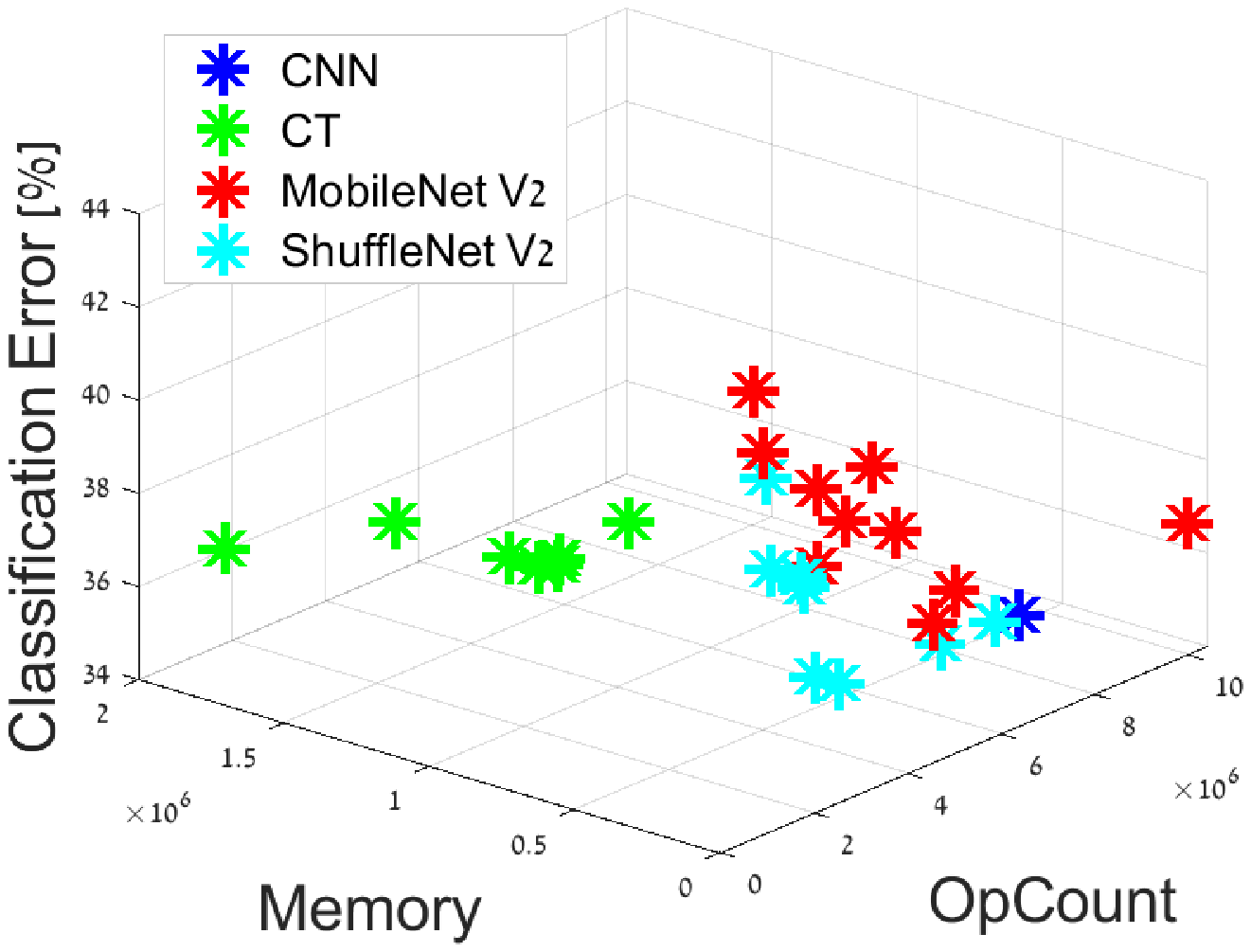}}%
\end{tabular}%
\caption{Performance trade-offs of the deep CT using an architectures of
six-layers. \textbf{(a)} The Speed:Error trade-off for several families of
models applied to the CIFAR-10 dataset. The models are based on the CNN in
Table~\protect\ref{table:NetworkArch_4}, CT networks, MobileNet V2~%
\protect\cite{sandler2018mobilenetv2}, and ShuffleNet V2~\protect\cite%
{ma2018shufflenet}. We considered architectures using up to $10^{7}$
operations. \textbf{(b)} Speed:Error trade-off on the CIFAR-100 dataset.
\textbf{(c)} Speed:Error:Memory trade-off on CIFAR-10. \textbf{(d)}
Speed:Error:Memory trade-off on CIFAR-100. CT-networks provide better
Speed:Error trade-off in this domain while requiring additional memory. The
speed is measured by the MACC operation count. The memory is given by the
number of parameters, while the error is the classification percentage
error. }
\label{fig:TradeOffs6}
\end{figure*}
\begin{table}[tbh]
\centering%
\begin{tabular}{c|ccc|ccc}
\hline
\textbf{Error} & \multicolumn{3}{|c|}{\textbf{Speedup}} &
\multicolumn{3}{|c}{\textbf{Memory ratio}} \\
\lbrack \%] & CT & \cite{sandler2018mobilenetv2} & \cite{ma2018shufflenet} &
CT & \cite{sandler2018mobilenetv2} & \cite{ma2018shufflenet} \\ \hline\hline
\multicolumn{7}{c}{CIFAR-10} \\ \hline
10.6 & 4.91 & 0.72 & 1.47 & 0.09 & 2.67 & 2.86 \\
11.6 & 5.69 & 1.59 & 3.51 & 0.14 & 5.53 & 8.12 \\
12.6 & 14.29 & 3.42 & 6.74 & 0.45 & 6.76 & 26.46 \\ \hline
\multicolumn{7}{c}{CIFAR-100} \\ \hline
36.6 & 4.91 & 0.72 & 2.80 & 0.09 & 2.67 & 6.77 \\
38.1 & 5.69 & 1.59 & 3.51 & 0.14 & 5.53 & 8.12 \\
39.6 & 13.74 & 3.42 & 6.74 & 0.28 & 6.76 & 26.46 \\ \hline
\end{tabular}%
\caption{Speed and memory trade-offs comparing the proposed CT, MobileNet2~%
\protect\cite{sandler2018mobilenetv2} and ShuffleNet2~\protect\cite%
{ma2018shufflenet}, using the CIFAR-10 and CIFAR-100 datasets. Each row
relates to a particular classification percentage error and shows the
corresponding speed and memory trade-offs. Speedup and memory ratios are
computed by $\frac{R_{CNN}}{R_{Net}}$ where $R_{CNN}$ is the resources
required by the baseline CNN, as in Table~\protect\ref{table:DCT2-results},
and $R_{Net}$ is the network's requirements. }
\label{Tab:TradeOff_Table}
\end{table}

\subsection{Trade-offs comparison: Accuracy, speed, and memory}

\label{subsec:TradeOffs}

We studied the Error:Speed and Error:Speed:Memory trade-offs enabled by the
CT networks and compared them to those of contemporary state-of-the-art
schemes. The experiments were conducted using the CIFAR10 and CIFAR100
datasets. Seven models of the CT network were trained, with operation count
budgets $3.1X-12X$ lower than those of the baseline CNN in Table~\ref%
{table:NetworkArch_4} ($7.54\cdot 10^{6}$). Same as the baseline CNN, the CT
models had six layers. The variations were derived by changing the
hyperparameters $M$, $K$, and the number of channels at intermediate layers.
\begin{table*}[tbh]
\centering%
\begin{tabular}{l|c|c|c|c|c|c|c|c|c|c}
\hline
\textbf{Dataset} & \textbf{Deep CT} & \textbf{BC} & \textbf{XN} & \textbf{%
TABCNN } & \textbf{DSQ 2 bits } & \textbf{DSQ 4 bits } & \textbf{DSQ 6 bits }
& \textbf{MobileNet} & \textbf{ShuffleNet} & \textbf{CNN} \\ \hline\hline
\textbf{CASIA 25K} & \textbf{87.21} & 85.36 & 86.286 & 86.95 & 87.05 & 87.15
& 87.19 & 87.20 & 87.186 & 89.96 \\
\textbf{CASIA 250K} & \textbf{90.66} & 87.95 & 88.35 & 88.89 & 90.06 & 90.27
& 90.35 & 90.63 & 90.54 & 93.25 \\ \hline
\end{tabular}%
\caption{The accuracy percentage of Deep CT, CNN, MobileNet V2~\protect\cite%
{sandler2018mobilenetv2} and ShuffleNet V2~\protect\cite{ma2018shufflenet}
when applied to subsets of the CASIA-WebFace dataset.}
\label{tab:DeepCTTable}
\end{table*}

We compare the trade-off of the CT networks with those of the MobileNet V2~%
\cite{sandler2018mobilenetv2} and ShuffleNet V2~\cite{ma2018shufflenet}
efficient CNN architectures. The first is based on inverted bottleneck
residual blocks, where depth-wise convolutions are performed on high-depth
representations, and the second used blocks in which half of the channels
are processed, followed by a channel shuffle operation. We used the code
provided by the corresponding authors\footnote{%
https://github.com/xiaochus/MobileNetV2} \footnote{%
https://github.com/TropComplique/shufflenet-v2-tensorflow} and adapted it to
the relevant low-compute regime by adjusting their hyperparameters. For
MobileNet, the networks were generated by adjusting the expansion parameter $%
t$ (Table 1 in~\cite{sandler2018mobilenetv2}), the number of channels in
intermediate representations, and the number of bottleneck modules in a
block (between $1$ and $2$). For ShuffleNet, the networks were generated by
changing the number of channels in the intermediate layers and the number of
bottleneck modules in a block (between $1$ and $4$).

Error:Speed results are shown in Figs. \ref{fig:TradeOffs6}a-b for CT
networks and competing lean network approaches. These plots show that the
proposed CT networks significantly outperform the baseline CNN and competing
architectures. To enable a more explicit assessment, Table~\ref%
{Tab:TradeOff_Table} left-hand side shows the speedup achieved by the most
efficient networks in each framework, for several error values on the
CIFAR-10 and CIFAR-100 datasets. In the high-accuracy end, while losing at
most $1\%$ error w.r.t. to the baseline CNN, the CT network enables $4.91X$
acceleration over the CNN and an acceleration of $1.75-6.8X$ over competing
methods. When more significant errors are allowed ($2.7\%$ for CIFAR-10, $%
4\% $ for CIFAR-100), the CT networks provide accelerations of $13X-14X$
over the baseline CNN, significantly outperforming the $6.7X$ acceleration
obtained by the best competitor~\cite{ma2018shufflenet}. The baseline CNN
uses $7.54\cdot 10^{7}$ MACC operations, $181\cdot 10^{6}$ parameters, and
obtains errors of $9.88$ and $35.58$ on CIFAR-10 and CIFAR-100,
respectively. The tables for CIFAR-10 and CIFAR-100 are similar, as the same
architectures are preferable for both datasets in each category of error and
network type.

Figures~\ref{fig:TradeOffs6}c-d depict the Error:Speed:Memory trade-off for
CT networks and other networks, using the CIFAR-10 and CIFAR-100 datasets,
respectively. The memory requirements of the most efficient networks for
each accuracy level are listed on the right-hand side of Table~\ref%
{Tab:TradeOff_Table}, as compression ratios with respect to the baseline CNN.
Clearly, the better Error:Speed trade-off of the CT networks comes at the
expense of a larger memory footprint (additional parameters). This is
expected, as in CT transformations with memory-based retrieval replacing
dot-product operations, the compute:memory ratio is much lower. For
instance, in the high-accuracy regime, using the CIFAR-10 dataset allows a $%
4.9X$ speedup over the baseline CNN while requiring $10.7X$ more parameters.
However, the size of excess memory can be controlled by limiting $K$
(for which the memory demand is exponential) to moderate values. Table~\ref%
{Tab:TradeOff_Table} shows that the fastest CT network enabling CIFAR-10
error of 12.6\% provides an $14.3X$ acceleration while requiring only $2.2X$
additional parameters over the baseline CNN.

\subsection{A face recognition test-case}

\label{sec:WebCasia}

We experimented with CT networks for face recognition using the
CASIA-WebFace dataset \cite{CASIA-WebFace}. Two datasets of $64\times 64$
extracted face images were formed. First, we randomly drew a subset of 50
classes (each with $5000$ images), 250K images overall, to evaluate the
proposed CT on a \textit{large-scale} dataset. We also randomly drew 50
identities and a class of negatives (not belonging to the 50 identities),
each consisting of $500$ images. This exemplifies a typical IoT application,
where face recognition is an enabler of personalized IoT services in smart
homes, as well as security cameras. For both datasets, we trained a six-layer CT network and a CNN comparable in terms of the number of
layers and intermediate tensor width. The architectures used are detailed in
Table~\ref{table:webcasia_arch} in the appendix. The CT accuracy was
compared to the baseline CNN, to the CNN versions produced by the
compression techniques used in Section~\ref{subsec:SimilarRepDim}, and to
the efficient CNN frameworks~\cite{sandler2018mobilenetv2,ma2018shufflenet}
used in Section~\ref{subsec:TradeOffs}. To compare with the six layers CNN
and CT networks, the MobileNet2 and ShuffleNet2 networks were constructed
with six blocks: an initial convolution layer and a bottleneck block in the
high resolution, followed by two bottleneck blocks in the second and third
resolutions. The number of maps at each intermediate representation was
identical to the widths used by the CT and CNN networks in Table~\ref%
{table:webcasia_arch}. The results are presented in Table~\ref%
{tab:DeepCTTable}. With $25K$ and $250K$ data samples, the CT network
outperformed all accelerated methods, with results comparable to MobileNet
V2, and was only outperformed by conventional CNN. In particular, the $%
250K$ set results show that the CT can be effectively applied to large
datasets.

\section{Summary and future work}

\label{sec:Discussion}

In this work, we introduced a novel deep learning framework that utilizes
indices computation and table access instead of dot-products. The
experimental successful validation of our framework implies that both
optimization and generalization are not uniquely related to the properties
and dynamics of dot-product neuron interaction and that having such
neuronal ingredients is an unnecessary condition for successful deep
learning. Our analysis and experimental results show that the suggested
framework outperforms conventional CNNs in terms of computational
complexity:capacity ratio. Empirically, we showed that it enables improved
speed:accuracy trade-off for the low-compute inference regime at the cost of
additional processing memory. Such an approach is applicable to a gamut of
applications that are either low-compute or are to be deployed on
laptops/tablets that are equipped with large memory, but lack dedicated GPU
hardware. Future work will extend the applicability of CT networks using an
improved GPU-based training procedure. In this context, it should be noted
that only basic CT networks were introduced and tested in this work, while
competing CNN architectures enjoy a decade of intensive empirical research.
The fact that CT networks were found competitive in this context is therefore
highly encouraging. Combining CT networks with CNN optimization techniques
(batch/layer normalization), regularization techniques (dropout), and
architectural benefits (Residuals connections) is yet to be tested. Finally,
to cope with the memory costs, methods for discretization/binarization of
the model's parameters should be developed, similarly to CNN discretization
algorithms.
\begin{table*}[tbph]
\centering%
\begin{tabular}{ll|cc|ccc}
\hline
\multicolumn{2}{c|}{\textbf{Deep CT}} & \textbf{Parameters} &
\multicolumn{1}{|c|}{\textbf{Output}} & \textbf{Convolution} &
\multicolumn{1}{|c}{\textbf{Parameters}} & \multicolumn{1}{|c}{\textbf{Output%
}} \\
\multicolumn{2}{c|}{\textbf{Layer}} & $l,K,M,D_{o}$ & \multicolumn{1}{|c|}{}
& \textbf{Layer} & \multicolumn{1}{|c}{$l,D_{o}$} & \multicolumn{1}{|c}{} \\
\hline\hline
1 & CT Layer 1 & $9,10,10,100$ & $20\times 20\times 100$ &
\multicolumn{1}{|l}{Conv,ReLU} & $l=5,D_{o}=100$ & $24\times 24\times 100$
\\
2 & Average Pooling & $l=7$ & $14\times 14\times 100$ & \multicolumn{1}{|l}{
Max Pooling} & $l=5,Stride=2$ & $10\times 10\times 100$ \\
3 & CT Layer 2 & $9,10,10,10$ & $6\times 6\times 10$ & \multicolumn{1}{|l}{
Conv,ReLU} & $l=5,D_{o}=10$ & $6\times 6\times 10$ \\
4 & Average Pooling & $l=6$ & $1\times 1\times 10$ & \multicolumn{1}{|l}{Max
Pooling+FC} & $l=6$ & $1\times 1\times 10$ \\
5 & SoftMax &  & $1$ & \multicolumn{1}{|l}{SoftMax} &  & $1$ \\ \hline
\end{tabular}%
\caption{The two-layers Deep CT and CNN networks applied to the MNIST
dataset. $l$ and $K$ denote the size of the patch and the number of bits, $M$ is
the number of convolutional tables and $D_{o}$ is the number of output maps.
}
\label{table:NetworkArch_2}
\end{table*}

\section{Appendix A}

The architectures used in Section~\ref{subsec:SimilarRepDim} for networks
consisting of $2$ and $4$ layers are described in Tables~\ref%
{table:NetworkArch_2} and \ref{table:NetworkArch_4}, respectively. The
architectures used for CASIA-Webface data in Section~\ref{subsec:TradeOffs}
are detailed in Table~\ref{table:webcasia_arch}.
\begin{table*}[tbh]
\centering%
\begin{tabular}{cl|cc|ccc}
\hline
\multicolumn{2}{c|}{\textbf{Deep CT}} & \textbf{Parameters } &
\multicolumn{1}{|c|}{\textbf{Output}} & \textbf{Convolution} &
\multicolumn{1}{|c}{\textbf{Parameters}} & \multicolumn{1}{|c}{\textbf{Output%
}} \\
\multicolumn{2}{c|}{\textbf{Layer}} & $l,K,M,D_{o}$ & \multicolumn{1}{|c|}{}
& \textbf{Layers} & \multicolumn{1}{|c}{} & \multicolumn{1}{|c}{} \\
\hline\hline
1 & CT Layer 1 & $7,8,8,32$ & $26\times 26\times 32$ & \multicolumn{1}{|l}{
Conv,ReLU} & $l=5,D_{o}=32$ & $28\times 28\times 32$ \\
2 & Average Pooling & $l=3$ & $24\times 24\times 32$ & \multicolumn{1}{|l}{
Max Pooling} & $l=3,Stride=2$ & $13\times 13\times 32$ \\
3 & CT Layer 2 & $7,8,8,32$ & $18\times 18\times 32$ & \multicolumn{1}{|l}{
Conv,ReLU} & $l=3,D_{o}=32$ & $11\times 11\times 32$ \\
4 & Average Pooling & $l=3$ & $16\times 16\times 32$ & \multicolumn{1}{|l}{
Max Pooling} & $l=3$ & $9\times 9\times 32$ \\
5 & CT Layer 3 & $7,8,8,64$ & $10\times 10\times 64$ & \multicolumn{1}{|l}{
Conv,ReLU} & $l=3,D_{o}=64$ & $7\times 7\times 64$ \\
6 & Average Pooling & $l=3$ & $8\times 8\times 64$ & \multicolumn{1}{|l}{Max
Pooling} & $l=3$ & $5\times 5\times 64$ \\
7 & CT Layer 4 & $7,8,8,10$ & $2\times 2\times 10$ & \multicolumn{1}{|l}{
Conv,ReLU} & $l=3,D_{o}=10$ & $3\times 3\times 10$ \\
8 & Average Pooling & $l=2$ & $1\times 1\times 10$ & \multicolumn{1}{|l}{Max
Pooling + FC} & $l=3$ & $1\times 1\times 10$ \\
9 & SoftMax &  &  & \multicolumn{1}{|l}{SoftMax} &  & $1$ \\ \hline
\end{tabular}%
\caption{The four-layer Deep CT and CNN networks applied to the CIFAR10,
CIFAR100, and SVHN datasets. The CNN binarization schemes were applied to
the CNN architecture, where $l$ and $K$ denote the patch size and the number
of bits, $M$ is the number of convolutional tables and $D_{o}$ is the number
of output maps.}
\label{table:NetworkArch_4}
\end{table*}
\begin{table*}[tbh]
\centering%
\begin{tabular}{clccccc}
\hline
\multicolumn{2}{|c}{\textbf{Deep CT}} & \multicolumn{1}{|c}{\textbf{%
Parameters }} & \multicolumn{1}{|c}{\textbf{Output}} & \multicolumn{1}{|c}{%
\textbf{Convolution}} & \multicolumn{1}{|c}{\textbf{Parameters}} &
\multicolumn{1}{|c|}{\textbf{Output}} \\
\multicolumn{2}{|c}{\textbf{Layer}} & \multicolumn{1}{|c}{$l,K,M,D_{o}$} &
\multicolumn{1}{|c}{} & \multicolumn{1}{|c}{\textbf{Layers}} &
\multicolumn{1}{|c}{} & \multicolumn{1}{|c|}{} \\ \hline\hline
1 & CT Layer 1 & \multicolumn{1}{|c}{$9,8,8,32$} & $56\times 56\times 32$ &
\multicolumn{1}{l}{Conv,ReLU} & $l=3,D_{o}=32$ & $62\times 62\times 32$ \\
2 & Average Pooling & \multicolumn{1}{|c}{$l=5$} & $52\times 52\times 32$ &
\multicolumn{1}{l}{Max Pooling} & $l=3,Stride=2$ & $30\times 30\times 32$ \\
3 & CT Layer 2 & \multicolumn{1}{|c}{$9,8,8,32$} & $44\times 44\times 32$ &
\multicolumn{1}{l}{Conv,ReLU} & $l=3,D_{o}=32$ & $28\times 28\times 32$ \\
4 & Average Pooling & \multicolumn{1}{|c}{$l=5$} & $40\times 40\times 32$ &
\multicolumn{1}{l}{Max Pooling} & $l=2,Stride=2$ & $14\times 14\times 32$ \\
5 & CT Layer 3 & \multicolumn{1}{|c}{$9,8,8,64$} & $32\times 32\times 64$ &
\multicolumn{1}{l}{Conv,ReLU} & $l=3,D_{o}=64$ & $12\times 12\times 64$ \\
6 & Average Pooling & \multicolumn{1}{|c}{$l=3$} & $30\times 30\times 64$ &
\multicolumn{1}{l}{Max Pooling} & $l=3$ & $10\times 10\times 64$ \\
7 & CT Layer 4 & \multicolumn{1}{|c}{$9,8,8,128$} & $22\times 22\times 128$
& \multicolumn{1}{l}{Conv,ReLU} & $l=3,D_{o}=128$ & $8\times 8\times 128$ \\
8 & Average Pooling & \multicolumn{1}{|c}{$l=3$} & $20\times 20\times 128$ &
\multicolumn{1}{l}{Max Pooling} & $l=3$ & $6\times 6\times 128$ \\
9 & CT Layer 5 & \multicolumn{1}{|c}{$9,8,8,256$} & $12\times 12\times 256$
& \multicolumn{1}{l}{Conv,ReLU} & $l=3,D_{o}=256$ & $5\times 5\times 256$ \\
10 & Average Pooling & \multicolumn{1}{|c}{$l=3$} & $10\times 10\times 256$
& \multicolumn{1}{l}{Max Pooling} & $l=3$ & $3\times 3\times 256$ \\
11 & CT Layer 6 & \multicolumn{1}{|c}{$9,8,8,50$} & $2\times 2\times 50$ &
\multicolumn{1}{l}{Conv,ReLU} & $l=3,D_{o}=50$ & $1\times 1\times 50$ \\
12 & Average Pooling & \multicolumn{1}{|c}{$l=2$} & $1\times 1\times 50$ &
\multicolumn{1}{l}{Fully Connected} &  & $1\times 1\times 50$ \\
13 & SoftMax &  &  & \multicolumn{1}{l}{SoftMax} &  & $1$ \\ \hline
\end{tabular}%
\caption{The Deep CT and CNN architectures applied to the subsets of the
CASIA-WebFace dataset. Where $l$ and $K$ denote the patch size and the
number of bits, $M$ is the number of convolutional tables, and $D_{o}$ is
the number of output maps.}
\label{table:webcasia_arch}
\end{table*}

\FloatBarrier

\bibliographystyle{plain}
\bibliography{CTEbib}

\end{document}